%% file: neurips_2026.tex
\documentclass{article}

\PassOptionsToPackage{numbers}{natbib}

\usepackage[main, final]{neurips_2026}

\usepackage[utf8]{inputenc} 
\usepackage[T1]{fontenc}    
\usepackage{hyperref}       
\usepackage{url}            
\usepackage{booktabs}       
\usepackage{amsfonts}       
\usepackage{nicefrac}       
\usepackage{microtype}      
\usepackage{xcolor}         
\usepackage{wrapfig}
\usepackage{bm}
\usepackage{amsmath}
\usepackage{amssymb}
\usepackage{mathtools}
\usepackage{amsthm}
\usepackage{float}

\usepackage[capitalize,noabbrev]{cleveref}

\theoremstyle{plain}
\usepackage{hyperref}
\usepackage{url}
\usepackage{wrapfig}
\usepackage{makecell}
\usepackage{amsthm}
\usepackage{comment}
\usepackage{refcount}
\usepackage{graphicx}
\usepackage{multirow}
\usepackage{amsmath, amssymb}
\usepackage[utf8]{inputenc} 
\usepackage[T1]{fontenc}    
\usepackage{hyperref}       
\usepackage{url}            
\usepackage{booktabs}       
\usepackage{amsfonts}       
\usepackage{nicefrac}       
\usepackage{microtype}      
\usepackage{xcolor}         
\usepackage{bbm}
\usepackage{subcaption}

\newcolumntype{?}{!{\vrule width 1pt}}
\newcommand{\argmin}{\mathop{\mathrm{argmin}}} 
\newcommand{\argmax}{\mathop{\mathrm{argmax}}} 
\newtheorem{theorem}{Theorem}[section]

\newtheorem{corollary}[theorem]{Corollary}

\newtheorem{assumption}[theorem]{Assumption}
\usepackage[textsize=tiny]{todonotes}

\title{MahaVar: OOD Detection via Class-wise Mahalanobis Distance Variance under Neural Collapse}

\author{%
  Donghwan Kim \\
  Department of Industrial Engineering\\
  Yonsei University\\
  Seoul, South Korea \\
  \texttt{cain0709@yonsei.ac.kr} \\
  \And
  Hyunsoo Yoon$^\dagger$ \\
  Department of Industrial Engineering\\
  Yonsei University\\
  Seoul, South Korea \\
  \texttt{hs.yoon@yonsei.ac.kr} \\
}

\begin{document}

\maketitle

\begin{abstract}
Out-of-distribution (OOD) detection is a critical component for ensuring the reliability of deep neural networks in safety-critical applications. In this work, we present a key empirical observation: for in-distribution (ID) samples, class-wise Mahalanobis distances exhibit a pronounced sharp minimum structure, where the distance to the nearest class is small while distances to all other classes remain large, resulting in high variance across classes. In contrast, OOD samples tend to exhibit a less pronounced sharp minimum structure, producing comparatively lower variance across classes. We further provide a theoretical analysis grounding this observation in Neural Collapse geometry: under relaxed Neural Collapse assumptions on within-class compactness and inter-class separation, ID samples are shown to structurally exhibit high class-wise distance variance, offering a theoretical basis for its use as an OOD score. Motivated by this observation and its theoretical backing, we propose MahaVar, a simple and effective post-hoc OOD detector that augments the Mahalanobis distance with a class-wise distance variance term. Following the OpenOOD v1.5 benchmark protocol, MahaVar achieves state-of-the-art performance on CIFAR-100 and ImageNet, with consistent improvements in both AUROC and FPR@95 over existing Mahalanobis-based methods across all benchmarks.
\end{abstract}

\section{Introduction}
\label{Intro}

Deploying deep neural networks in real-world environments inevitably raises a critical challenge: models trained under a closed-world assumption often encounter test samples that fall outside their training distribution, leading to overconfident and potentially catastrophic predictions. Out-of-distribution (OOD) detection addresses this challenge by identifying such samples before they cause erroneous decisions, making it an essential component for building reliable and trustworthy AI systems \citep{yang2024generalized}. The importance of OOD detection has been widely recognized across a range of safety-critical domains, including autonomous driving \citep{shoeb2025out}, medical imaging \citep{hong2024out}, and industrial inspection \citep{han2022out}, where the consequences of model failure can be severe. OOD detection methods generally build upon pre-trained neural networks and can be broadly categorized into two families. The first family consists of training-based methods~\citep{ming2023how, du2022vos, zhu2023diversified, tao2023nonparametric}, which require either access to auxiliary OOD datasets or modification of the training objective to explicitly shape the decision boundary between ID and OOD data. However, such approaches are inherently constrained by the quality and coverage of the available outlier data, and retraining a model for every new deployment scenario is often impractical. The second family consists of post-hoc methods, which leverage the network's representations~\citep{sun2022out,ren2021simple,lee2018simple, muller2025mahalanobis++, park2023nearest, liu2025detecting}, output logits~\citep{hendrycks2017a, liu2020energy, wang2022vim, Liu_2023_CVPR}, or activation shaping~\citep{sun2021react,xu2023vra,djurisic2023extremely,xu2024scaling,kim2026anticoherence,hassan2026davis,Hassan_2026_CVPR} of a pre-trained classifier without any additional training or access to OOD data. This training-free nature makes post-hoc methods significantly more practical for real-world deployment, where modifying the underlying model is often infeasible.

Among post-hoc methods, Mahalanobis distance-based detectors~\citep{lee2018simple, muller2025mahalanobis++} have demonstrated strong OOD detection performance by computing, for each test sample, the minimum Mahalanobis distance to the class-conditional Gaussian distributions fitted on training data, where each class is characterized by its mean vector and a shared tied covariance matrix. Building upon this framework, we identify a key empirical observation: for ID samples, the sorted class-wise Mahalanobis distances exhibit a pronounced sharp minimum structure, in which the distance to the nearest class is substantially smaller than the distances to all remaining classes. In contrast, OOD samples tend to exhibit a less sharp minimum structure across classes, resulting in lower variance. This structural asymmetry between ID and OOD samples motivates the use of class-wise distance variance as a complementary and discriminative OOD signal.

We explain this phenomenon through the lens of Neural Collapse~\citep{papyan2020prevalence}, a geometric phenomenon observed during the terminal phase of deep network training, wherein penultimate features of each class collapse toward their class mean, and the class means arrange themselves into a simplex Equiangular Tight Frame (ETF) with equal inter-class distances and angles. Under relaxed Neural Collapse assumptions, we theoretically analyze the class-wise variance of Mahalanobis distances for ID samples and provide a principled explanation for why the sharp minimum structure emerges. We further discuss conditions under which OOD samples exhibit lower class-wise distance variance, offering a geometric justification for the proposed score.

Motivated by this observation and its theoretical grounding, we propose MahaVar, a simple yet effective post-hoc OOD score that builds upon Mahalanobis++~\citep{muller2025mahalanobis++} by augmenting its nearest class Mahalanobis distance with a class-wise distance variance term, addressing the remaining room for improvement in this state-of-the-art framework. The key insight is that by jointly considering both where a sample lies relative to its nearest class and how dispersed its distances are across all classes, MahaVar captures the richer geometric structure of the feature space in a way that neither term alone can achieve. To validate the effectiveness of MahaVar, we conduct extensive experiments following the benchmark protocol proposed in OpenOOD v1.5~\citep{zhang2024openood}, across three ID settings including CIFAR-10, CIFAR-100~\citep{krizhevsky2009learning}, and ImageNet~\citep{deng2009imagenet}, under both Near-OOD and Far-OOD evaluation settings. MahaVar consistently outperforms all Mahalanobis-based baselines across all three benchmarks, and achieves state-of-the-art performance among all evaluated methods on CIFAR-100 and ImageNet. Notably, on the large-scale ImageNet benchmark averaged over three backbone architectures (ResNet-50~\citep{he2016deep}, Swin-B~\citep{liu2021swin}, and ViT-B~\citep{dosovitskiy2021an}), MahaVar achieves the best average AUROC and FPR@95 across all evaluated baselines, demonstrating that the proposed variance term provides consistent gains regardless of the underlying backbone architecture.

In summary, this work makes the following threefold contributions.

\begin{itemize}
    \item We present a key empirical observation that class-wise Mahalanobis distances exhibit a pronounced sharp minimum structure for in-distribution samples, while OOD samples exhibit a less sharp minimum structure across classes. To our knowledge, this distance-rank structure has not been previously characterized or exploited for OOD detection.
  \item We provide a \textbf{theoretical analysis grounded in Neural Collapse geometry} that characterizes the conditions under which in-distribution samples exhibit strictly higher class-wise distance variance than OOD samples, offering a principled justification for the use of variance as a discriminative OOD signal.
  \item We propose \textbf{MahaVar}, a simple post-hoc OOD detector that augments the standard Mahalanobis score with a class-wise distance variance term. Following the OpenOOD v1.5 benchmark protocol, MahaVar achieves state-of-the-art performance on CIFAR-100 and ImageNet, and consistently outperforms all Mahalanobis-based baselines across all three benchmarks in both AUROC and FPR@95.
\end{itemize}

\section{Related Work}
\label{related_work}

\subsection{Post-hoc OOD Detection Methods}
Post-hoc methods detect OOD samples using a pre-trained classifier without any additional training or access to OOD data, and can be broadly categorized into representation-based, logit-based, and activation shaping methods.

\noindent\textbf{Representation-based methods} leverage the geometry of the penultimate layer features to construct OOD scores. \citet{lee2018simple} proposed the Mahalanobis distance to class-conditional Gaussian distributions fitted on penultimate layer features, establishing the foundation for distance-based OOD detection. \citet{ren2021simple} extended this with a relative Mahalanobis distance that normalizes class-conditional distances against a global reference distribution to improve near-OOD detection. \citet{muller2025mahalanobis++} further improved robustness by applying L2 normalization to penultimate layer features prior to distance computation, aligning the feature distributions more closely with the Gaussian assumptions underlying the Mahalanobis distance. To avoid such parametric assumptions, \citet{sun2022out} proposed a non-parametric KNN-based score that directly measures distances to training samples in the penultimate layer feature space, and \citet{park2023nearest} built on this by weighting the energy score using cosine similarity to the $k$-nearest neighbors. \citet{liu2025detecting} leveraged Neural Collapse geometry to show that centered ID features cluster near their predicted class weight vectors, proposing an angle-based proximity score complemented by feature norm filtering. \citet{liu2024fast} proposed fDBD, which measures the distance from the test sample to the classifier's decision boundary. CADRef \citep{ling2025cadref} partitions feature components based on the sign alignment between their weighted contribution and the classifier weight, then normalizes the class-conditional deviation within each partition by a logit-based score.

\noindent\textbf{Logit-based methods} derive OOD scores directly from classifier outputs. \citet{hendrycks2017a} introduced the maximum softmax probability as a simple yet effective baseline. Building upon this, \citet{liu2020energy} proposed the energy score, which aggregates all class logits and provides better theoretical grounding than softmax-based scores. \citet{Liu_2023_CVPR} further improved sensitivity by applying a generalized entropy over the top-$M$ sorted class probabilities, capturing small entropy variations among the most confident predictions. \citet{wang2022vim} constructs a virtual OOD class logit from the norm of the feature residual against a principal subspace, and combines it with the original class logits via softmax to produce a unified OOD score.

\noindent\textbf{Activation shaping methods} modify penultimate layer activations at inference time to amplify the gap between ID and OOD samples. \citet{sun2021react} proposed clipping abnormally large activations in the penultimate layer, observing that OOD samples tend to produce extreme activation patterns. \citet{djurisic2023extremely} extended this idea with ASH, which prunes activations below a threshold and rescales the remaining ones based on per-sample statistics. \citet{xu2023vra} further generalized activation shaping by applying both lower and upper truncation to suppress abnormally small and large values. Subsequently, \citet{xu2024scaling} showed that scaling alone, without pruning, achieves competitive performance while avoiding the accuracy degradation associated with activation removal. Furthermore, OptFS \citep{zhao2024towards} learns an optimal feature scaling vector directly from the training data, while AdaSCALE \citep{regmi2026adascale} dynamically adjusts the percentile threshold per sample based on an OOD likelihood estimated through gradient-based input perturbation.

\subsection{Neural Collapse}
\label{neural_collapse}
Neural Collapse is a geometric phenomenon first observed by \citet{papyan2020prevalence} during the terminal phase of deep network training, wherein the penultimate layer features and the linear classification head simultaneously converge to a highly structured configuration. Formally, let $\mathbf{h}^i_c$ denote the penultimate layer feature of the $i$-th training sample with class label $c$, $\bm{\mu}_c = \text{Avg}_i\{\mathbf{h}^i_c\}$ the class mean, $\bm{\mu}_G = \text{Avg}_{i,c}\{\mathbf{h}^i_c\}$ the global mean, and $\mathbf{w}_c$ the weight vector of class $c$ in the linear classification head. Neural Collapse comprises four inter-related properties:
\begin{itemize}
    \item \textbf{NC1 (Within-class variability collapse):} The within-class covariance $\Sigma_W = \text{Avg}_{i,c}\{(\mathbf{h}^i_c - \bm{\mu}_c)(\mathbf{h}^i_c - \bm{\mu}_c)^\top\} \to 0$, meaning that features of samples within the same class collapse toward their class mean.
    \item \textbf{NC2 (Convergence to a simplex ETF):} The centered class means $\{\bm{\mu}_c - \bm{\mu}_G\}$ converge to a simplex Equiangular Tight Frame (ETF), satisfying
    \begin{equation}
        \frac{(\bm{\mu}_c - \bm{\mu}_G)^\top(\bm{\mu}_{c'} - \bm{\mu}_G)}{\|\bm{\mu}_c - \bm{\mu}_G\|_2 \|\bm{\mu}_{c'} - \bm{\mu}_G\|_2} \to \frac{C}{C-1}\delta_{c,c'} - \frac{1}{C-1},
    \end{equation}
    where $\delta_{c,c'}$ is the Kronecker delta. This implies that all centered class means $\{\bm{\mu}_c - \bm{\mu}_G\}$ become equinorm and maximally separated in angle, forming the most spread-out configuration achievable for $C$ vectors in high-dimensional space.
    \item \textbf{NC3 (Convergence to self-duality):} The classifier weight vectors align with the centered class means up to a scaling factor:
    \begin{equation}
        \frac{\mathbf{w}_c}{\|\mathbf{w}_c\|_2} - \frac{\bm{\mu}_c - \bm{\mu}_G}{\|\bm{\mu}_c - \bm{\mu}_G\|_2} \to 0,
    \end{equation}
    establishing a dual relationship between the feature geometry and the classification head.
    \item \textbf{NC4 (Simplification to nearest class center):} The classifier simplifies to a nearest class center decision rule:
    \begin{equation}
        \argmax_{c \in \{1,\ldots,C\}}\ \mathbf{w}_c^\top \mathbf{h} + b_c \to \argmin_{c \in \{1,\ldots,C\}}\ \|\mathbf{h} - \bm{\mu}_c\|_2,
    \end{equation}
    meaning that classification based on the learned linear head becomes equivalent to assigning the sample to its nearest class mean.
\end{itemize}

While complete Neural Collapse requires prolonged training beyond zero training error, its convergence trends and corresponding geometric patterns have been consistently observed across diverse architectures and classification tasks before full convergence~\citep{liu2025detecting}. This practical prevalence has motivated its application to OOD detection~\citep{liu2025detecting, ammar2024neco}, where the structured geometry of ID features serves as a reference for detecting anomalous samples.

\section{Sharp Minimum Structure of Class-wise Mahalanobis Distances}
\label{observation_sharp}
\subsection{Mahalanobis Distance-based OOD Detection}

The Mahalanobis distance-based detector~\citep{lee2018simple} leverages class-conditional statistics estimated from the penultimate layer features of the training set $\{(\mathbf{x}_i, y_i)\}_{i=1}^{n}$. For each class $c$, the class mean $\hat{\bm{\mu}}_c$ is computed over the penultimate features $\phi(\cdot)$ of all training samples belonging to that class. A single covariance matrix $\hat{\Sigma}$ is shared across all classes and estimated by pooling the deviations of the penultimate features from their respective class means:

\begin{equation}
\label{eq:mahala_mean_cov}
\hat{\bm{\mu}}_c = \frac{1}{N_c} \sum_{i: y_i = c} \phi(\mathbf{x}_i), \qquad
\hat{\Sigma} = \frac{1}{N} \sum_{c=1}^{C} \sum_{i: y_i = c} \left(\phi(\mathbf{x}_i) - \hat{\bm{\mu}}_c\right)\left(\phi(\mathbf{x}_i) - \hat{\bm{\mu}}_c\right)^\top
\end{equation}

where $N_c$ and $N$ denote the number of samples in class $c$ and in total, respectively. For a test sample $\mathbf{x}$, the squared Mahalanobis distance between its penultimate feature $\phi(\mathbf{x})$ and each class mean is computed as:

\begin{equation}
d_{\text{Maha},c}(\mathbf{x}) = \left(\phi(\mathbf{x}) - \hat{\bm{\mu}}_c\right)^\top \hat{\Sigma}^{-1} \left(\phi(\mathbf{x}) - \hat{\bm{\mu}}_c\right)
\end{equation}

This formulation rests on the assumption that the penultimate features of each class follow a class-conditional Gaussian distribution $\mathcal{N}(\hat{\bm{\mu}}_c, \hat{\Sigma})$ with a shared tied covariance matrix $\hat{\Sigma}$ across all classes. Intuitively, a test sample is more likely to be in-distribution if its penultimate feature is close to at least one class mean in the Mahalanobis sense. This motivates the following OOD score:

\begin{equation}
S(\mathbf{x}) = -\min_{c \in \{1, \ldots, C\}} d_{\text{Maha},c}(\mathbf{x})
\end{equation}

where a higher score indicates a greater likelihood of the sample being in-distribution. A test sample is classified as OOD when $S(\mathbf{x})$ falls below a threshold $T$, which is typically calibrated by fixing the true positive rate (TPR) at 95\% on the in-distribution validation set.

\citet{muller2025mahalanobis++} proposed Mahalanobis++, which further improves upon this formulation by applying L2 normalization to the penultimate features prior to computing the Mahalanobis distance, encouraging the normalized features to better satisfy the class-conditional Gaussian assumption. In this work, we build upon this L2-normalized Mahalanobis framework and demonstrate that there remains further room for improvement. Specifically, we observe that the class-wise Mahalanobis distances exhibit a sharp minimum structure that carries discriminative information beyond what the minimum distance alone captures, motivating our proposed method.

\subsection{Observation: Sharp Minimum in Class-wise Mahalanobis Distances}
\label{sec:obseravtion_sharp}
\begin{figure}[t]
    \centering
    \includegraphics[width=\textwidth]{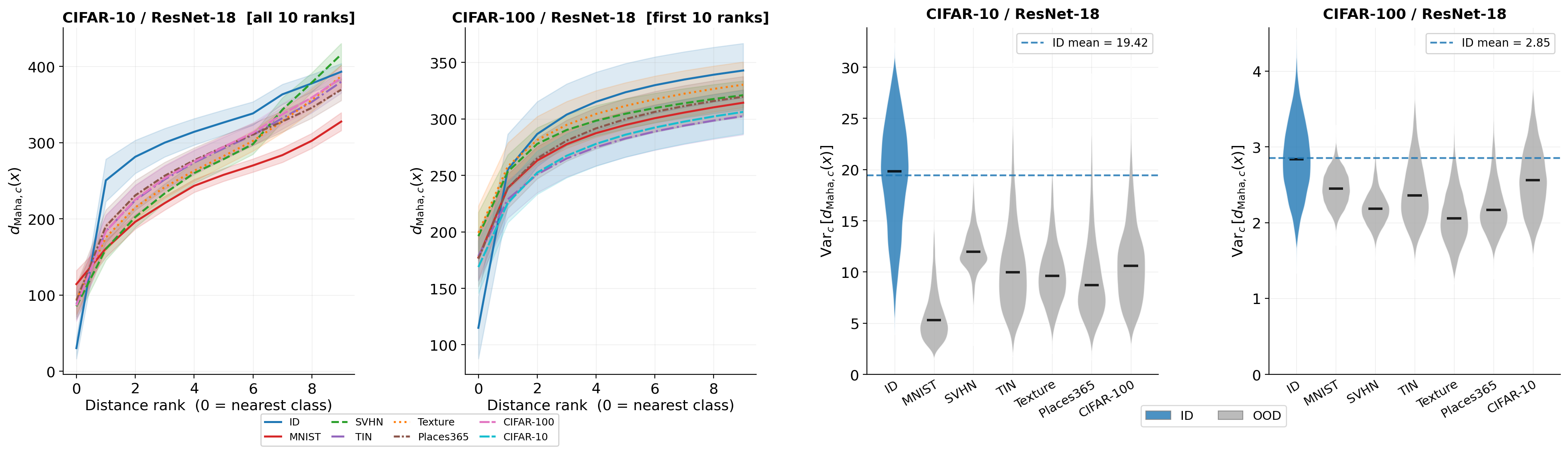}
    \caption{Sorted class-wise $d_{\text{Maha},c}(\mathbf{x})$ (left two panels) and the corresponding distribution of $\text{Var}_c[d_{\text{Maha},c}(\mathbf{x})]$ for ID and OOD samples (right two panels) of Mahalanobis++~\citep{muller2025mahalanobis++} on CIFAR-10 and CIFAR-100 as ID datasets, evaluated on OOD datasets following the OpenOOD v1.5 benchmark protocol~\citep{zhang2024openood} with ResNet-18 as the backbone. In the sorted distance plots, the x-axis represents the Mahalanobis distance rank, where rank 0 corresponds to the nearest class mean, and shaded regions indicate $\pm 0.5\sigma$ across samples at each rank. In the variance distribution plots, the dashed blue line indicates the mean variance of ID samples, and black bars indicate the median of each distribution.}
    \label{fig:cifar_maha_combined}
\end{figure}

To motivate our method, we conduct an empirical analysis using Mahalanobis++ under the OpenOOD v1.5 benchmark protocol. Specifically, we consider CIFAR-10 and CIFAR-100 as ID datasets, each with six OOD datasets. For each ID setting, we use a ResNet-18 pretrained on the respective ID dataset and extract the L2-normalized penultimate features. The class-wise means $\hat{\bm{\mu}}_c$ and tied covariance matrix $\hat{\Sigma}$ are then estimated following Equation~\ref{eq:mahala_mean_cov}. For each test sample, we compute the Mahalanobis distance to all class means and sort them in ascending order. 

As shown in the left two panels of Figure~\ref{fig:cifar_maha_combined}, which plot the sorted class-wise Mahalanobis distances for ID and OOD samples, the results reveal a consistent pattern: ID samples exhibit a substantially lower Mahalanobis distance at rank 0 compared to OOD samples, followed by a sharp increase that quickly plateaus at higher ranks. This indicates that ID samples tend to lie close to a specific class mean while remaining far from all others, whereas OOD samples maintain comparatively large distances to all class means without a clear dominant minimum. 

Furthermore, this sharp minimum structure directly implies elevated class-wise distance variance for ID samples: one distance is substantially smaller while the rest plateau at similarly large values, yielding high variance, whereas the less sharp structure of OOD samples yields comparatively lower variance, as verified in the right two panels of Figure~\ref{fig:cifar_maha_combined}. There, OOD samples consistently exhibit lower class-wise distance variance than ID samples across all OOD datasets on both CIFAR-10 and CIFAR-100, with the ID mean variance substantially exceeding the median of every OOD distribution. This separation suggests that $\text{Var}_c[d_{\text{Maha},c}(\mathbf{x})]$ naturally captures the class-wise distance structure of ID samples, making it a reliable complementary OOD signal alongside the minimum distance term. Additional results for the original Mahalanobis detector~\citep{lee2018simple} on CIFAR-10/100 and ImageNet, as well as Mahalanobis++ on ImageNet, are provided in Appendix~\ref{appendix:additional_sorted}.

While existing Mahalanobis-based methods either focus solely on the nearest class distance~\citep{lee2018simple, muller2025mahalanobis++} or consider a global distributional perspective~\citep{ren2021simple}, the observed class-wise distance structure suggests unexploited discriminative information, motivating the theoretical analysis and method proposed in the following section.


\section{Neural Collapse Perspective on ID/OOD Variance Separation}
\label{sec:neural_collapse_explanation}

In this section, we provide a theoretical analysis of the class-wise distance variance structure under the framework of Neural Collapse~\citep{papyan2020prevalence}. As introduced in Section~\ref{neural_collapse}, Neural Collapse describes a set of geometric properties that emerge in deep neural networks during the terminal phase of training. Notably, Neural Collapse tendencies have been observed even before full convergence, suggesting that its geometric structure is a general property of well-trained classifiers rather than a phenomenon exclusive to the terminal phase. 

However, perfect Neural Collapse is rarely achieved in practice, as it requires strict convergence conditions that are difficult to satisfy with finite data and standard training procedures. We therefore adopt a relaxed formulation, retaining only the geometric properties pertaining to the penultimate feature space while excluding assumptions on the classifier weights. Specifically, we discard NC3 (convergence to self-duality) and NC4 (simplification to nearest class-center), as these involve the classifier weight vectors, which are not accessible in a post-hoc detection setting and are not required for our analysis. Instead, we rely solely on weakened versions of NC1 and NC2, which characterize the geometry of the penultimate features and class means, respectively.
We formalize these as the following:

\begin{assumption}[Weak NC1 and NC2]
\label{assumption:weak_nc1_nc2}
Let $\phi: \mathcal{X} \rightarrow \mathbb{R}^d$ denote the penultimate feature extractor of a pretrained network. We make the following two assumptions:

\textbf{(Weak NC1)} For each class $c \in \{1, \ldots, C\}$, all $\textbf{x}_i\in\mathcal{X}_{in}$ with $y_i=c$ lie within an $\ell_2$-ball of radius $\epsilon_c$ centered at the class mean $\bm{\mu}_c$:
\begin{equation}
\|\phi(\bm{x}_i) - \bm{\mu}_c\|_2 \leq \epsilon_c \quad \forall i : y_i = c
\end{equation}
This relaxes the strict NC1 condition, which requires within-class variability to vanish (i.e., $\epsilon_c \to 0$), by permitting a bounded deviation of radius $\epsilon_c > 0$.

\textbf{(NC2)} The centered class means $\{\bm{\mu}_c - \bm{\mu}_G\}_{c=1}^C$ form a simplex ETF under the $\ell_2$-distance, i.e., they are equinorm with $\|\bm{\mu}_c - \bm{\mu}_G\|_2 = R$ for all $c$, and equiangular with  inter-class distance $K$:
\begin{equation}
\|\bm{\mu}_c - \bm{\mu}_{c'}\|_2 = K \quad \forall c \neq c'
\end{equation}
Furthermore, we assume $\epsilon = \max_c \epsilon_c < K/2$.
\end{assumption}

Assumption~\ref{assumption:weak_nc1_nc2} states that the class means form a simplex ETF, while allowing each in-distribution sample to deviate from its class mean by at most $\epsilon_c$ in the $\ell_2$ sense. The constraint $\epsilon < K/2$ ensures that each sample remains closer to its own class mean than to any other class mean, preserving the nearest-class assignment implied by the ETF structure. Note that a simplex ETF with $C$ classes requires a penultimate layer dimension of at least $C-1$. For ImageNet ($C=1000$), this condition is not satisfied by Transformer-based backbones such as ViT-B ($d=768$), which prevents the class means from forming a complete simplex ETF and consequently weakens the variance separation between ID and OOD samples observed in Appendix~\ref{appendix:additional_sorted}. Under this assumption, we establish bounds on the class-wise $\ell_2$ distance variance between ID samples and their class means in Theorem~\ref{theorem:id_variance_bound}.

\begin{theorem}[Bounds on ID Class-wise $\ell_2$ Distance Variance]
\label{theorem:id_variance_bound}
Let $\phi: \mathcal{X} \rightarrow \mathbb{R}^d$ be the penultimate feature extractor of a pretrained network, where $\mathcal{X}_{in} \subset \mathcal{X}$ denotes the in-distribution input space. Suppose Assumption~\ref{assumption:weak_nc1_nc2} holds for $\phi$ on all $\mathbf{x} \in \mathcal{X}_{in}$. Then for any test ID sample $\mathbf{x} \in \mathcal{X}_{in}$, the class-wise squared $\ell_2$ distance variance satisfies:
\begin{equation*}
\frac{(C-1)K^4}{C^2}\left(1-\gamma\right)^2 \leq \mathrm{Var}_c\left[\|\phi(\mathbf{x}) - \bm{\mu}_c\|_2^2\right] \leq \frac{(C-1)K^4}{C^2}\left(1+\gamma\right)^2 + \frac{2\epsilon^2 K^2(C-2)}{C}
\label{eq:id_variance_bound}
\end{equation*}
where $\gamma \triangleq \epsilon\sqrt{\frac{2C}{K^2(C-1)}}$. In particular, as $\epsilon \to 0$, $\gamma \to 0$ and both bounds converge to $\frac{(C-1)K^4}{C^2}$, which depends solely on the inter-class distance $K$ and the number of classes $C$.
\end{theorem}

Theorem~\ref{theorem:id_variance_bound} establishes lower and upper bounds on the
class-wise distance variance of ID samples under the ETF geometry. Building on the
empirical success of Mahalanobis++ in the L2-normalized feature space and for
theoretical tractability, we conduct the subsequent analysis under
Assumption~\ref{assumption:weak_nc1_nc2} posed directly on the L2-normalized
features $\hat{\phi}$, assuming $\sum_c \bm{\mu}_c = \mathbf{0}$ so that the ETF
structure applies directly to $\bm{\mu}_c$.
Theorem~\ref{theorem:variance_separation} further characterizes the conditions under
which the class-wise distance variance of L2-normalized ID features strictly exceeds
that of OOD samples.

\begin{theorem}[Sufficient Condition for Strictly Higher Variance in ID Samples]
\label{theorem:variance_separation}
Let $\hat{\phi}(\mathbf{x}) = (\phi(\mathbf{x}) - \bm{\mu}_G) / \|\phi(\mathbf{x}) - \bm{\mu}_G\|_2$ be the centered L2-normalized penultimate feature extractor of a pretrained network. Suppose Assumption~\ref{assumption:weak_nc1_nc2} holds for $\hat{\phi}$ on all $\mathbf{x} \in \mathcal{X}_{in}$, where $R = \|\bm{\mu}_c\|_2 \leq 1$ and $\sum_c \bm{\mu}_c = \mathbf{0}$. For any OOD sample $\mathbf{x}^{out} \in \mathcal{X} \setminus \mathcal{X}_{in}$ satisfying $\max_c |\hat{\phi}(\mathbf{x}^{out})^\top \bm{\mu}_c| < \frac{R(R-\epsilon)}{\sqrt{C-1}}$, we have $\mathrm{Var}_c\left[\|\hat{\phi}(\mathbf{x}) - \bm{\mu}_c\|_2^2\right] > \mathrm{Var}_c\left[\|\hat{\phi}(\mathbf{x}^{out}) - \bm{\mu}_c\|_2^2\right]$ for any ID sample $\mathbf{x} \in \mathcal{X}_{in}$.
\end{theorem}

Theorem~\ref{theorem:variance_separation} shows that when an OOD sample maintains sufficient angular separation from all class means in terms of absolute cosine similarity, the class-wise distance variance of ID samples is guaranteed to strictly exceed that of OOD samples, confirming that the variance term carries meaningful discriminative information for OOD detection. We further strengthen this result in Corollary~\ref{corollary:variance_separation_always} by introducing an additional assumption that the within-class deviation of ID samples lies within the subspace spanned by the class means. Under this assumption, the class-wise distance variance of ID samples is shown to universally dominate that of any OOD sample, regardless of the OOD sample's proximity to the class means.

\begin{corollary}[Guaranteed Variance Separation under ETF Subspace Alignment]
\label{corollary:variance_separation_always}
Under the assumptions of Theorem~\ref{theorem:variance_separation}, suppose further that for any ID sample $\mathbf{x} \in \mathcal{X}_{in}$ from class $c^*$, the within-class deviation satisfies $\hat{\phi}(\mathbf{x}) - \bm{\mu}_{c^*} \in \mathrm{span}(\bm{\mu}_1, \ldots, \bm{\mu}_C)$. Then for any OOD sample $\mathbf{x}^{out} \in \mathcal{X} \setminus \mathcal{X}_{in}$, $\mathrm{Var}_c\left[\|\hat{\phi}(\mathbf{x}) - \bm{\mu}_c\|_2^2\right] \geq \mathrm{Var}_c\left[\|\hat{\phi}(\mathbf{x}^{out}) - \bm{\mu}_c\|_2^2\right]$ holds, with equality if and only if $\hat{\phi}(\mathbf{x}^{out}) \in \mathrm{span}(\bm{\mu}_1, \ldots, \bm{\mu}_C)$.
\end{corollary}

Corollary~\ref{corollary:variance_separation_always} establishes that under the ETF subspace alignment assumption, ID samples universally exhibit higher class-wise distance variance than OOD samples. A notable implication of this result is that even when a near-OOD sample shares the same nearest class distance as an ID sample, the variance term retains discriminative power as long as the near-OOD sample does not lie within the ETF subspace spanned by the class means, a signal that the nearest class distance alone cannot capture. The assumption of Corollary ~\ref{corollary:variance_separation_always} corresponds to $\rho(x)=1$ for ID samples, where $\rho(x)=\|P\hat{\phi}(x)\|_2^2$ and $P$ denotes the projection onto the span of the class means, and it should be regarded as an idealized limit. What we verify empirically is the weaker on-average ordering $\rho_{\mathrm{ID}} > \rho_{\mathrm{OOD}}$, which holds in 22 of 22 backbone and dataset combinations with $D \ge C$ (Appendix~\ref{appendix:effect_of_var}). Motivated by the empirical findings in Section~\ref{observation_sharp} and the theoretical analysis presented in this section, we propose \textbf{MahaVar}, a post-hoc OOD score that augments the standard Mahalanobis distance with a class-wise distance variance term:
\begin{equation}
\label{eq:mahavar}
S(\mathbf{x}) = -\min_{c \in \{1, \ldots, C\}} d_{\mathrm{Maha},c}(\hat{\phi}(\mathbf{x})) + \alpha \cdot \mathrm{Var}_c\left[d_{\mathrm{Maha},c}(\hat{\phi}(\mathbf{x}))\right]
\end{equation}
where $\hat{\phi}(\mathbf{x}) = \phi(\mathbf{x}) / \|\phi(\mathbf{x})\|_2$ denotes the L2-normalized penultimate feature and $\alpha \geq 0$ is a scalar hyperparameter that controls the contribution of the variance term. Note that while the theoretical analysis is conducted under the $\ell_2$ distance for analytical tractability, an analogous bound holds for the Mahalanobis distance under the assumption that the within-class deviation in Mahalanobis sense satisfies $\epsilon_M < K_M/2$, where $\epsilon_M$ and $K_M$ denote the Mahalanobis counterparts of $\epsilon$ and $K$, respectively. All proofs are provided in Appendix~\ref{appendix:proof_theorems}. Also, the practical implementation of MahaVar omits the centering step in $\hat{\phi}$; we discuss the rationale and its empirical implications in Appendix~\ref{appendix:effect_centering}.

\section{Experiment}
\label{sec:experiment}

\begin{table}[t]
\centering
\caption{OOD detection performance on the ImageNet benchmark, averaged over three backbones (ResNet-50, Swin-B, ViT-B). Results are reported in AUROC and FPR@95 for each OOD dataset. Best and second best results are highlighted in \textbf{bold} and \underline{underline}, respectively.}
\label{tab:main_results_imagenet}
\resizebox{0.9\textwidth}{!}{
\begin{tabular}{lcccccccccccc}
\toprule
\multirow{3}{*}{Method} & \multicolumn{4}{c}{Near-OOD} & \multicolumn{6}{c}{Far-OOD} & \multicolumn{2}{c}{\multirow{2}{*}{Avg}} \\
\cmidrule(lr){2-5} \cmidrule(lr){6-11}
& \multicolumn{2}{c}{SSB-Hard} & \multicolumn{2}{c}{NINCO} & \multicolumn{2}{c}{iNaturalist} & \multicolumn{2}{c}{Texture} & \multicolumn{2}{c}{OpenImage-O} & & \\
\cmidrule(lr){2-3} \cmidrule(lr){4-5} \cmidrule(lr){6-7} \cmidrule(lr){8-9} \cmidrule(lr){10-11} \cmidrule(lr){12-13}
& AUROC & FPR@95 & AUROC & FPR@95 & AUROC & FPR@95 & AUROC & FPR@95 & AUROC & FPR@95 & AUROC & FPR@95 \\
\midrule
MSP         & 70.83 & 83.15 & 78.83 & 73.64 & 87.68 & 51.71 & 82.56 & 61.60 & 83.87 & 61.97 & 80.75 & 66.41 \\
MaxLogit    & 67.47 & 83.60 & 74.89 & 74.46 & 85.10 & 52.86 & 82.30 & 55.16 & 80.60 & 60.71 & 78.07 & 65.36 \\
KL-matching & 69.38 & 82.16 & 80.59 & 71.49 & 90.00 & 46.72 & 85.14 & 61.63 & 87.05 & 58.53 & 82.43 & 64.11 \\
Energy      & 63.59 & 85.99 & 70.01 & 78.61 & 79.22 & 63.69 & 79.86 & 56.63 & 75.35 & 66.30 & 73.61 & 70.24 \\
GEN         & \underline{72.02} & 83.48 & 82.04 & 71.07 & 92.46 & 39.91 & 87.87 & 51.96 & 88.83 & 53.68 & 84.65 & 60.02 \\
ReAct       & 68.29 & 84.56 & 79.14 & 74.73 & 90.31 & 48.05 & 88.34 & 49.81 & 87.31 & 56.21 & 82.68 & 62.67 \\
ASH-S       & 48.18 & 92.34 & 44.57 & 87.33 & 40.89 & 70.66 & 49.34 & 69.37 & 42.95 & 76.56 & 45.19 & 79.25 \\
SCALE       & 63.33 & 86.53 & 69.71 & 77.14 & 78.24 & 56.94 & 82.04 & 49.29 & 77.73 & 60.58 & 74.21 & 66.10 \\
KNN         & 66.29 & 88.71 & 81.03 & 75.80 & 88.82 & 56.87 & 92.42 & 36.90 & 88.23 & 57.08 & 83.36 & 63.07 \\
NCI         & 67.75 & 86.91 & 80.79 & 74.83 & 91.50 & 49.59 & 89.47 & 49.00 & 89.91 & 54.39 & 83.88 & 62.94 \\
NCI+filter  & 69.56 & 84.96 & 81.89 & 71.74 & 92.85 & 40.23 & 90.67 & 42.30 & 90.89 & 47.95 & 85.17 & 57.43 \\
fDBD        & 68.64 & 86.32 & 81.51 & 73.82 & 92.02 & 47.02 & 90.02 & 47.55 & 90.40 & 53.06 & 84.52 & 61.55 \\
ViM         & 69.63 & 83.92 & 82.56 & 69.67 & 93.28 & 39.41 & 92.43 & \textbf{34.97} & 91.75 & 45.38 & 85.93 & 54.67 \\
NNGuide     & 71.67 & 84.11 & 83.21 & 70.79 & 93.48 & 39.39 & 92.25 & 37.71 & 91.44 & 48.96 & 86.41 & 56.19 \\
\midrule
Mahalanobis             & 64.46 & 85.03 & 77.99 & 73.77 & 84.90 & 47.86 & 89.14 & 51.38 & 84.70 & 56.27 & 80.24 & 62.86 \\
Mahalanobis++           & 71.50 & \underline{78.45} & \underline{85.56} & \underline{61.95} & \underline{95.69} & \underline{22.22} & \textbf{93.15} & \underline{35.26} & \underline{92.83} & \underline{39.24} & \underline{87.75} & \underline{47.42} \\
\textbf{MahaVar (Ours)} & \textbf{73.02} & \textbf{75.87} & \textbf{86.92} & \textbf{59.03} & \textbf{96.82} & \textbf{16.12} & \underline{93.14} & 35.83 & \textbf{93.33} & \textbf{36.83} & \textbf{88.64} & \textbf{44.74} \\
\bottomrule
\end{tabular}
}
\end{table}

\begin{table}[t]
\centering
\caption{OOD detection performance on CIFAR-10 and CIFAR-100 benchmarks. Results are reported in terms of AUROC and FPR@95, averaged over all OOD datasets for each ID setting. The best and second best results are highlighted in \textbf{bold} and \underline{underline}, respectively.}
\label{tab:main_results_cifar}
\begin{tabular}{lcccc}
\toprule
\multirow{2}{*}{Method} & \multicolumn{2}{c}{CIFAR-10} & \multicolumn{2}{c}{CIFAR-100} \\
\cmidrule(lr){2-3}\cmidrule(lr){4-5}
 & AUROC $\uparrow$ & FPR@95 $\downarrow$ & AUROC $\uparrow$ & FPR@95 $\downarrow$ \\
\midrule
MSP & 89.94 & 51.34 & 78.71 & 79.93 \\
MaxLogit & 90.02 & \underline{40.08} & 80.25 & 78.85 \\
KL-matching & 87.00 & 51.70 & 79.04 & 77.38 \\
Energy & 90.11 & \textbf{39.78} & 80.26 & 78.26 \\
GEN & 90.45 & 41.39 & 80.21 & 78.70 \\
ReAct & 89.70 & 42.95 & 80.32 & 78.27 \\
ASH-S & 85.60 & 44.45 & 81.65 & 74.96 \\
SCALE & 86.61 & 43.62 & 81.43 & 75.71 \\
KNN & \textbf{91.81} & 45.73 & 78.35 & 78.69 \\
NCI & 90.38 & 49.85 & 81.38 & 77.84 \\
NCI+filter & 90.38 & 49.85 & 81.38 & 77.84 \\
fDBD & 88.59 & 53.37 & 80.78 & 77.79 \\
ViM & 90.07 & 57.48 & 79.16 & 77.87 \\
NNGuide & 90.22 & 40.37 & 80.72 & 78.08 \\
\midrule
Mahalanobis & 88.12 & 66.50 & 66.85 & 90.68 \\
Mahalanobis++ & 90.80 & 49.64 & \underline{82.75} & \underline{70.33} \\
\textbf{MahaVar (Ours)} & \underline{91.12} & 45.56 & \textbf{83.80} & \textbf{68.03} \\
\bottomrule
\end{tabular}
\end{table}

\textbf{Experimental Setup.} We follow the benchmark protocol proposed in OpenOOD v1.5. For CIFAR-10 and CIFAR-100 as ID datasets, we use ResNet-18 as the backbone, with MNIST~\citep{deng2012mnist}, SVHN~\citep{netzer2011reading}, Textures~\citep{cimpoi2014describing}, and Places365~\citep{zhou2017places} as far-OOD datasets and TIN~\citep{le2015tiny} and CIFAR-100 and CIFAR-10 as near-OOD datasets, respectively. For ImageNet as the ID dataset, we evaluate across three backbone architectures: ResNet-50, Swin-B, and ViT-B, with iNaturalist~\citep{van2018inaturalist}, Textures, and OpenImage-O~\citep{wang2022vim} as far-OOD datasets and SSB-Hard~\citep{vaze2021open} and NINCO~\citep{bitterwolf2023or} as near-OOD datasets. For the CIFAR-10/100 ID settings, we use the pretrained models provided by OpenOOD v1.5. For the ImageNet ID setting, we use the pretrained models provided by torchvision~\citep{torchvision2016}. 

We compare MahaVar against a comprehensive set of post-hoc OOD detection baselines: Mahalanobis~\citep{lee2018simple}, Mahalanobis++~\citep{muller2025mahalanobis++}, KNN~\citep{sun2022out}, NCI~\citep{liu2025detecting}, NCI+filter~\citep{liu2025detecting}, ViM~\citep{wang2022vim}, fDBD~\citep{liu2024fast}, MSP~\citep{hendrycks2017a}, MaxLogit~\citep{hendrycks2022scaling}, KL-matching~\citep{hendrycks2022scaling}, Energy~\citep{liu2020energy}, GEN~\citep{Liu_2023_CVPR}, ReAct~\citep{sun2021react}, ASH-S~\citep{djurisic2023extremely}, SCALE~\citep{xu2024scaling}, and NNGuide~\citep{park2023nearest}. Performance is evaluated using AUROC and FPR@95. All hyperparameters are tuned on the validation sets provided by OpenOOD v1.5, with train, validation, and test splits following its data configuration. The main results for the ImageNet ID setting averaged over three backbone architectures and the CIFAR-10/100 ID settings are presented in Table~\ref{tab:main_results_imagenet} and Table~\ref{tab:main_results_cifar}, respectively. Further implementation details are provided in Appendix~\ref{appendix:implementation_details}. 

\textbf{Main Results.} According to Table~\ref{tab:main_results_imagenet}, MahaVar achieves the best AUROC and FPR@95 on all Near-OOD datasets and on two out of three Far-OOD datasets, with Texture being the only exception where MahaVar marginally underperforms Mahalanobis++, confirming that the variance term provides meaningful complementary signal beyond the nearest class distance alone. Notably, while SCALE and ASH-S achieve the highest performance on ResNet-50, their performance collapses dramatically on Transformer-based backbones, with average AUROC dropping to 63.57\% (SCALE) and 23.22\% (ASH-S) on Swin-B, and 68.64\% (SCALE) and 22.66\% (ASH-S) on ViT-B, respectively. In contrast, MahaVar achieves competitive performance on ResNet-50 and attains the best overall performance on Swin-B and ViT-B, consistently outperforming Mahalanobis++ across all backbone architectures and demonstrating stable gains regardless of the underlying feature extractor.

As shown in Table~\ref{tab:main_results_cifar}, MahaVar consistently outperforms all Mahalanobis-based baselines on both CIFAR settings on average. On CIFAR-100, MahaVar achieves the best overall AUROC and FPR@95 at 83.80\% and 68.03\%, respectively, outperforming all evaluated baselines. On CIFAR-10, while KNN achieves a higher AUROC, MahaVar outperforms all Mahalanobis-based methods including Mahalanobis++ in both metrics, confirming that the variance term provides reliable gains within the representation-based family. Taken together with the ImageNet results in Table~\ref{tab:main_results_imagenet}, MahaVar achieves state-of-the-art performance on the more challenging ID settings of CIFAR-100 and ImageNet, demonstrating that the variance term provides a reliable complementary signal across diverse benchmark scales. Appendix~\ref{appendix:additional_experiments} provides per-backbone results for the ImageNet benchmark, along with CIFAR-10/100 results against additional baselines.

\section{Ablation Study}

We conduct ablation studies to validate the design choices of MahaVar, where all ImageNet results are averaged over the three backbone architectures.

\textbf{Variance term and L2 normalization.} Adding the variance term to the standard Mahalanobis distance without L2 normalization already improves the average AUROC from 80.24\% to 86.76\% and reduces the average FPR@95 from 62.86\% to 52.33\%, which confirms that the variance term carries meaningful discriminative information. However, this variant still falls short of Mahalanobis++ (87.75\% AUROC and 47.42\% FPR@95), indicating that L2 normalization stabilizes the feature geometry for distance computation. MahaVar, which combines both components, achieves the best performance overall with 88.64\% AUROC and 44.74\% FPR@95 (Appendix \ref{appendix:effect_of_var}).

\textbf{Variance as a standalone score.} Using the variance term in isolation yields only 59.91\% AUROC and 88.90\% FPR@95, indicating that variance alone is insufficient as a standalone score and that the two terms are complementary rather than substitutable. A pointwise analysis further provides a theoretical justification for combining the two terms additively (Appendix~\ref{appendix:effect_of_var}).

\textbf{Higher-order moments.} Extending the score with a skewness term yields only marginal gains (88.72\% AUROC and 44.56\% FPR@95), suggesting that higher-order moments of the class-wise distance distribution contain additional discriminative information. However, since the gains are modest and the optimal sign of the skewness coefficient depends on the backbone, we adopt the variance-only formulation as our main method (Appendix~\ref{appendix:effect_of_higher_order}).

\textbf{Sensitivity to the hyperparameter $\alpha$.} The optimal $\alpha$ is 0.05 for ResNet-50, 0.1 for Swin-B, and 0.07 for ViT-B on the OpenImage-O validation set, and performance degrades as the variance term begins to dominate the nearest-class distance term beyond this range (e.g., the validation AUROC of ResNet-50 drops from 92.75\% to 89.30\% at $\alpha = 0.2$). This range can be reliably identified using the validation set provided by OpenOOD v1.5, making $\alpha$ selection straightforward in practice (Appendix~\ref{appendix:alpha_sensitivity}).

\textbf{Hyperparameter-free ratio variant.} We further explore a ratio-based variant that removes $\alpha$ entirely. On ResNet-50, this score reaches 88.40\% AUROC and 39.96\% FPR@95 compared with 88.24\% and 41.27\% for the tuned additive form, though establishing a comparable theoretical separation guarantee for it remains an open direction (Appendix~\ref{appendix:vard1}).

\textbf{Distance metrics.} Incorporating the variance term consistently improves performance regardless of the base metric, from 86.67\% to 87.50\% AUROC for L1, from 87.39\% to 88.02\% for L2, and from 87.75\% to 88.64\% for Mahalanobis, confirming the general utility of the variance signal. Among all combinations, the Mahalanobis distance achieves the best results owing to its ability to account for the covariance structure of the feature space (Appendix~\ref{appendix:distance_metric}).

\textbf{Number of nearest classes.} Performance generally improves as the number of nearest classes $K$ used for variance computation increases, with $K = 1000$ being optimal for ResNet-50 and Swin-B, while ViT-B peaks at $K = 500$ due to its insufficient penultimate layer dimension preventing complete ETF formation. Nevertheless, we adopt $K = 1000$ across all ImageNet backbones in the main experiments for consistency, and all available classes ($K = 10$ and $K = 100$, respectively) are used for CIFAR-10 and CIFAR-100 (Appendix~\ref{appendix:top_k_effect}).

\textbf{When does the variance augmentation help.} We provide a practical guideline based on the ratio of feature dimension to class count $d/C$. The AUROC gain of +1.19 on ResNet-50 ($d/C = 2.05$) and +1.20 on Swin-B ($d/C = 1.02$) shrinks to +0.30 on ViT-B ($d/C = 0.77$), consistent with the incomplete ETF formation observed for ViT-B. Beyond this ratio, we further discuss how the diagnostic statistics of \citet{papyan2020prevalence} can directly assess how closely a given ID dataset satisfies the underlying Neural Collapse assumptions (Appendix~\ref{appendix:additional_sorted}).

\textbf{Comparison with additional recent baselines.} Finally, we compare against additional recent baselines not included in the main experiments, including CADRef, LTS, OptFS, and AdaSCALE-A/L (Appendix~\ref{appendix:recent_baseline}). On ResNet-50, MahaVar trails activation-shaping and logit-scaling methods such as AdaSCALE, LTS, and CADRef by a narrow margin in AUROC, mirroring the gap observed against SCALE in the main ResNet-50 results. On Swin-B and ViT-B, however, this gap reverses, with MahaVar substantially outperforming all of these methods, consistent with the architecture-dependent trend observed throughout this analysis.

\section{Conclusion}
\label{conclusion}
We presented MahaVar, a simple and effective post-hoc OOD detector that augments the Mahalanobis distance with a class-wise distance variance term. The key insight is that class-wise distance variance serves as a geometric indicator of class alignment: under Neural Collapse geometry, ID samples concentrate near a specific class mean while remaining far from all others, and this sharp minimum structure directly induces higher class-wise distance variance than OOD samples exhibit. This theoretical grounding, combined with consistent empirical gains across diverse benchmarks and backbone architectures, suggests that the full class-wise distance structure contains richer discriminative information than the nearest-class distance alone can capture. We believe that extending the theoretical framework beyond Neural Collapse, incorporating higher-order moments such as skewness, developing a rigorous treatment of the non-centered feature setting, and establishing a theoretical separation guarantee for ratio-based scoring are promising directions for future work. The limitations of this work are discussed in Appendix~\ref{appendix:limitation}.

\begin{ack}
This work was supported by the Institute of Information \& Communications Technology Planning \& Evaluation (IITP) grant funded by the Korean government (MSIT) [grant number RS-2025-02305884] and the National Research Foundation of Korea (NRF) grant funded by the Korean government (MSIT) [grant number RS-2026-25619394].
\end{ack}

\bibliographystyle{plainnat}
\bibliography{references}

\newpage
\appendix

\section{Additional Sorted Class-wise Mahalanobis Distance Structures}
\label{appendix:additional_sorted}

This section presents additional sorted class-wise Mahalanobis distance structures and variance distributions to complement the observations in Section~\ref{sec:obseravtion_sharp}. Specifically, we include results on non-L2-normalized penultimate features for CIFAR-10/100 (Figure~\ref{fig:sorted_maha_cifar_raw}), and both L2-normalized and non-L2-normalized penultimate features for ImageNet (Figure~\ref{fig:sorted_maha_imagenet_l2} and Figure~\ref{fig:sorted_maha_imagenet_raw}, respectively). The sharp minimum structure of ID samples is consistently observed across all settings, confirming that the observation in Section~\ref{sec:obseravtion_sharp} is not a result of L2 normalization. Additionally, we provide the corresponding $\text{Var}_c[d_{\text{Maha},c}(x)]$ distributions for CIFAR-10/100 without L2 normalization (Figure~\ref{fig:var_dist_cifar_raw}), and for ImageNet with and without L2 normalization (Figure~\ref{fig:var_dist_imagenet_l2} and Figure~\ref{fig:var_dist_imagenet_raw}, respectively). Consistent with the main observations, ID samples exhibit higher class-wise distance variance than OOD samples across all settings, with the exception of Transformer-based backbones on ImageNet where the separation is less pronounced due to the insufficient penultimate layer 
dimension ($d = 768 < C - 1 = 999$ for ViT-B) that prevents complete ETF formation.
\begin{figure}[h]
    \centering
    \includegraphics[width=0.9\linewidth]{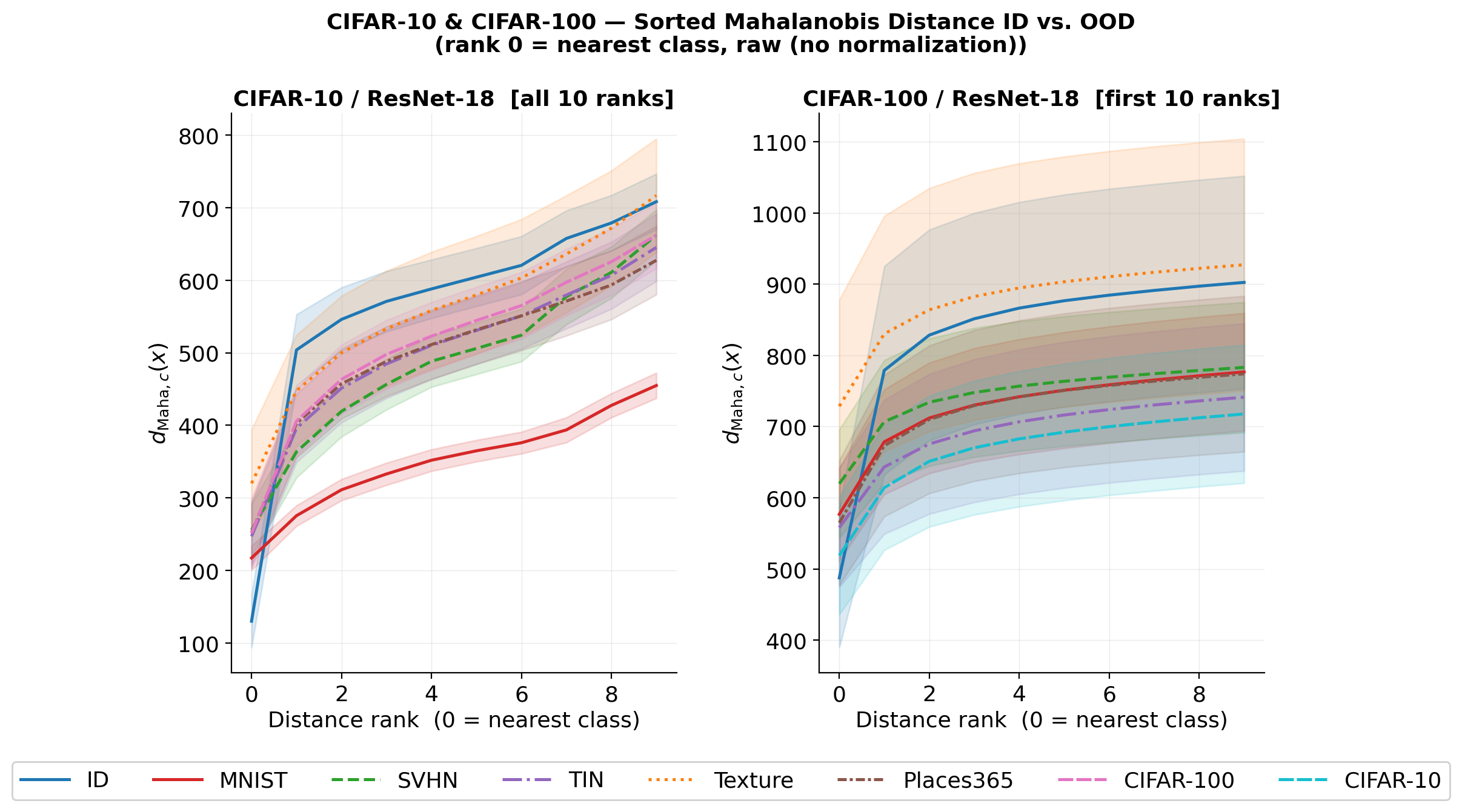}
    \caption{Sorted class-wise $d_{\text{Maha},c}(\mathbf{x})$ (without L2 normalization) on CIFAR-10 and CIFAR-100 as ID datasets, evaluated on OOD datasets following the OpenOOD v1.5 benchmark protocol with ResNet-18 as the backbone. The x-axis represents the Mahalanobis distance rank, where rank 0 corresponds to the nearest class mean. Shaded regions indicate $\pm 0.5\sigma$ across samples at each rank.}
    \label{fig:sorted_maha_cifar_raw}
\end{figure}

\begin{figure}[h]
    \centering
    \includegraphics[width=0.9\linewidth]{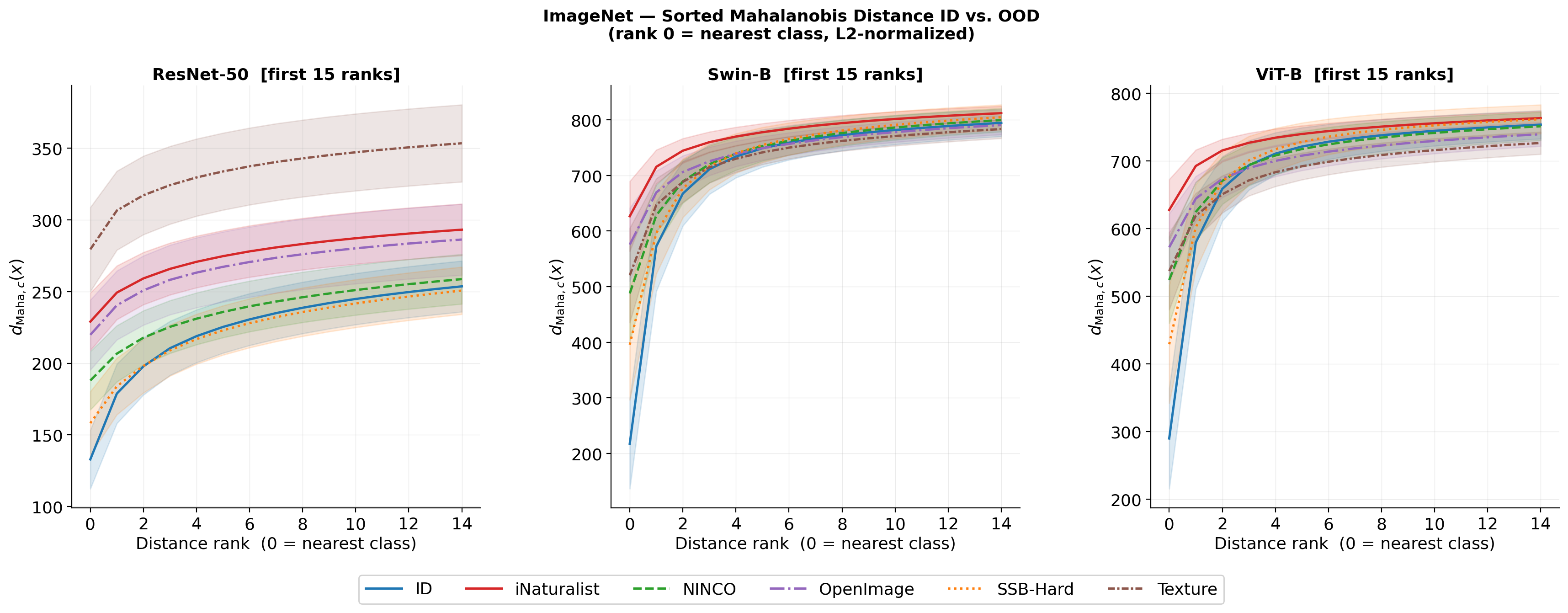}
    \caption{Sorted class-wise $d_{\text{Maha},c}(\mathbf{x})$ of Mahalanobis++ on ImageNet as the ID dataset, evaluated on five OOD datasets following the OpenOOD v1.5 benchmark protocol across three backbone architectures (ResNet-50, Swin-B, and ViT-B). The x-axis represents the Mahalanobis distance rank, where rank 0 corresponds to the nearest class mean. Shaded regions indicate $\pm 0.5\sigma$ across samples at each rank.}
    \label{fig:sorted_maha_imagenet_l2}
\end{figure}

\begin{figure}[h]
    \centering
    \includegraphics[width=0.9\linewidth]{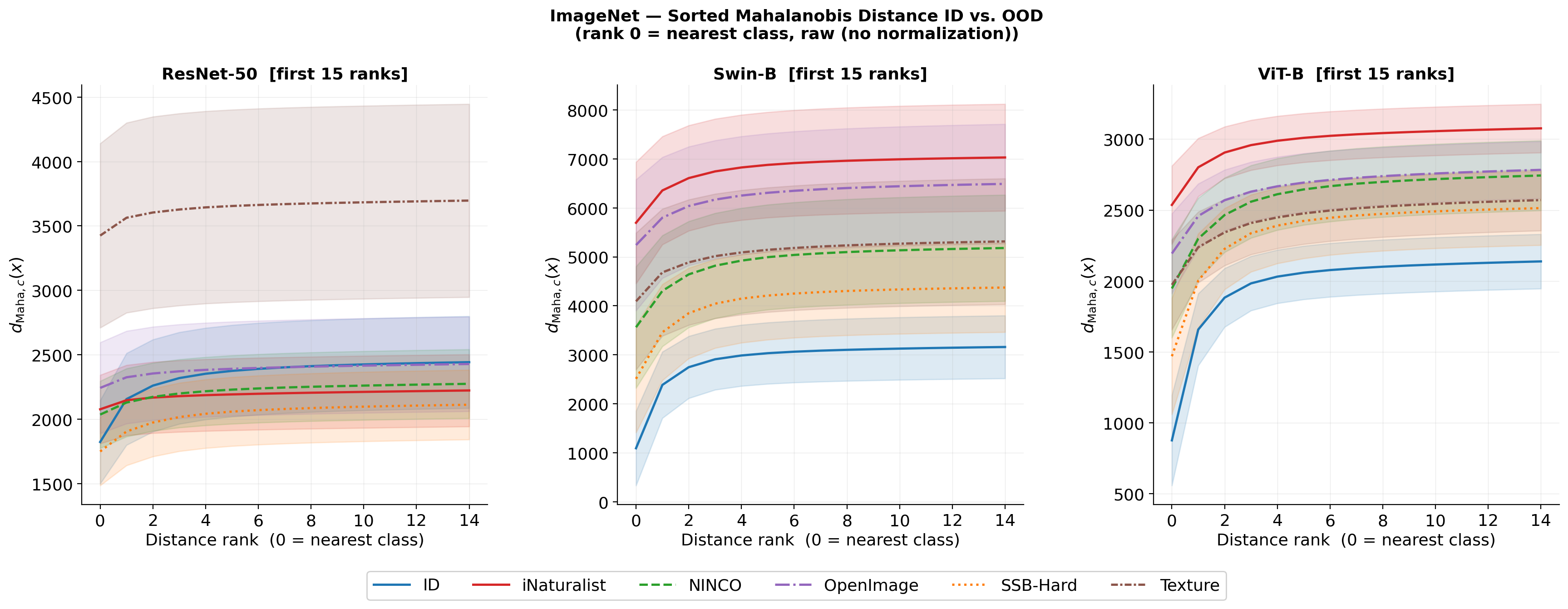}
    \caption{Sorted class-wise $d_{\text{Maha},c}(\mathbf{x})$ (without L2 normalization) on ImageNet as the ID dataset, evaluated on five OOD datasets following the OpenOOD v1.5 benchmark protocol across three backbone architectures (ResNet-50, Swin-B, and ViT-B). The x-axis represents the Mahalanobis distance rank, where rank 0 corresponds to the nearest class mean. Shaded regions indicate $\pm 0.5\sigma$ across samples at each rank.}
    \label{fig:sorted_maha_imagenet_raw}
\end{figure}

\begin{figure}[h]
    \centering
    \includegraphics[width=0.9\linewidth]{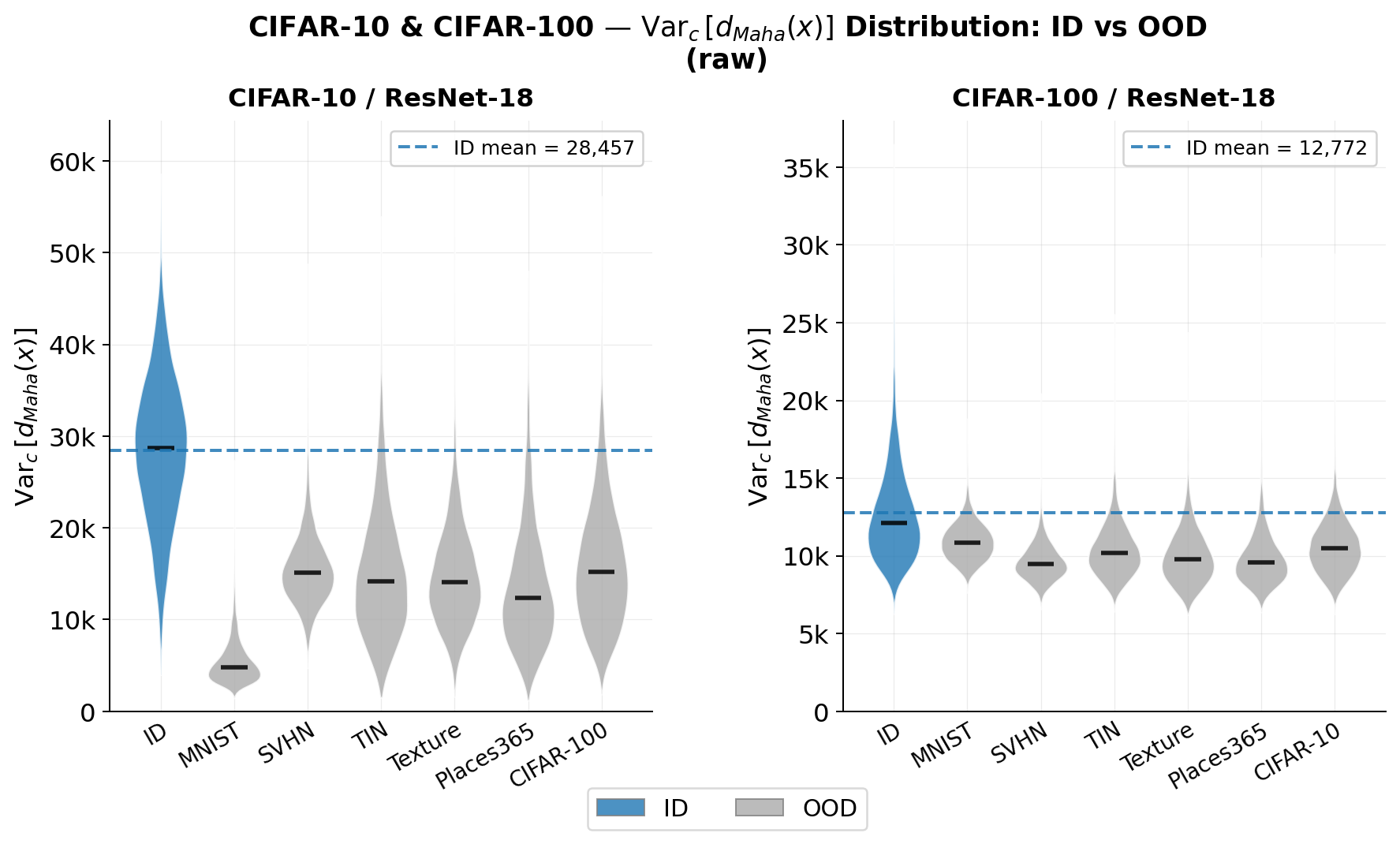}
    \caption{Distribution of $\text{Var}_c[d_{\text{Maha},c}(\mathbf{x})]$ for ID and OOD samples on CIFAR-10 and CIFAR-100 without L2 normalization, evaluated on OOD datasets following the OpenOOD v1.5 benchmark protocol with ResNet-18 as the backbone. The dashed blue line indicates the mean variance of ID samples. Black bars indicate the median of each distribution.}
    \label{fig:var_dist_cifar_raw}
\end{figure}

\begin{figure}[h]
    \centering
    \includegraphics[width=0.9\linewidth]{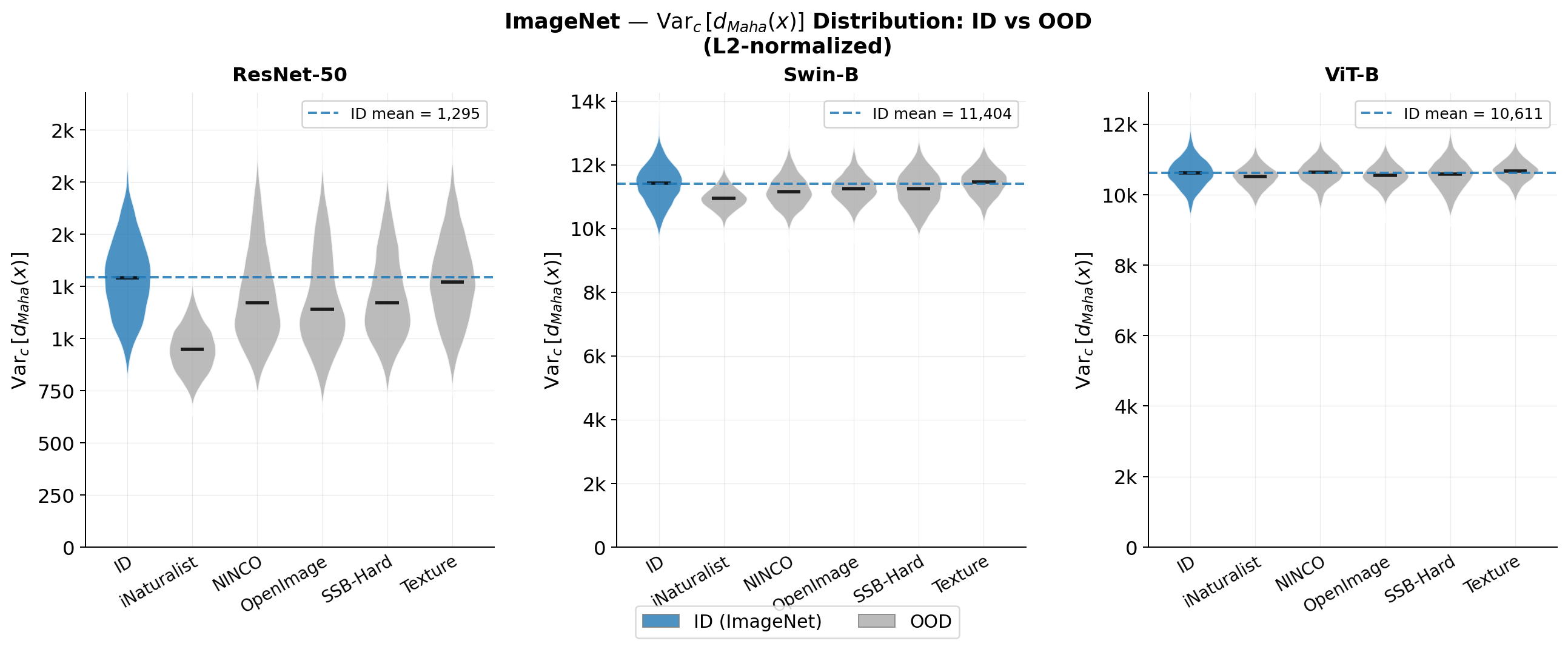}
    \caption{Distribution of $\text{Var}_c[d_{\text{Maha},c}(\mathbf{x})]$ for ID and OOD samples on ImageNet with L2 normalization following Mahalanobis++, evaluated on five OOD datasets following the OpenOOD v1.5 benchmark protocol across three backbone architectures (ResNet-50, Swin-B, and ViT-B). The dashed blue line indicates the mean variance of ID samples. Black bars indicate the median of each distribution.}
    \label{fig:var_dist_imagenet_l2}
\end{figure}

\begin{figure}[h]
    \centering
    \includegraphics[width=0.9\linewidth]{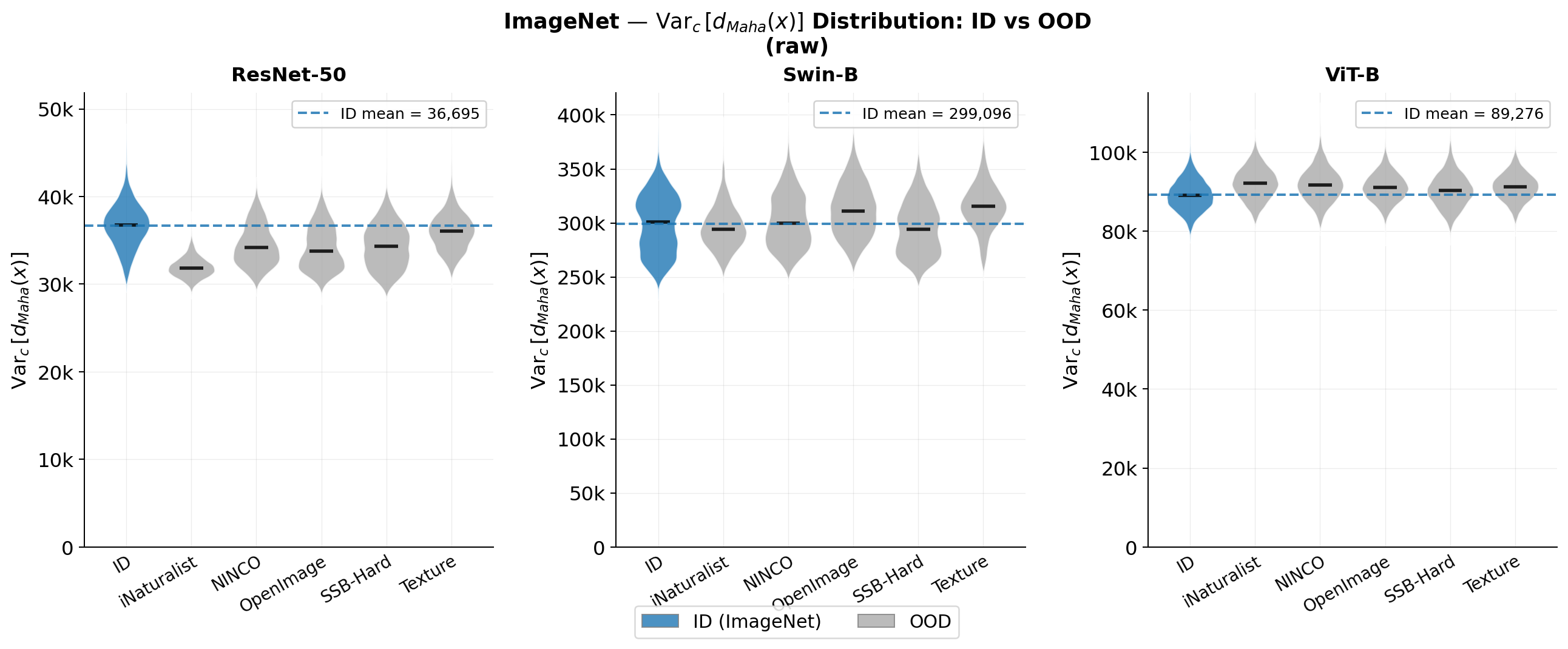}
    \caption{Distribution of $\text{Var}_c[d_{\text{Maha},c}(\mathbf{x})]$ for ID and OOD samples on ImageNet without L2 normalization, evaluated on five OOD datasets following the OpenOOD v1.5 benchmark protocol across three backbone architectures (ResNet-50, Swin-B, and ViT-B). The dashed blue line indicates the mean variance of ID samples. Black bars indicate the median of each distribution.}
    \label{fig:var_dist_imagenet_raw}
\end{figure}

\clearpage

Additionally, we provide a practical guideline for when the variance augmentation can be expected to help. Table~\ref{tab:dimension_gain} reports the AUROC gain of MahaVar over Mahalanobis++ on ImageNet for each backbone alongside its feature dimension $d$ relative to $C$.

\begin{table}[H]
\centering
\caption{Relationship between the ratio of feature dimension to class count ($d/C$) and the AUROC gain of MahaVar over Mahalanobis++ on ImageNet, averaged across all five OOD datasets.}
\label{tab:dimension_gain}
\begin{tabular}{lcccc}
\toprule
Backbone & $d$ & $C$ & $d/C$ & Mahalanobis++ $\to$ MahaVar (AUROC) \\
\midrule
ResNet-50 & 2048 & 1000 & 2.05 & 87.05 $\to$ 88.24 \; ($+1.19$) \\
Swin-B    & 1024 & 1000 & 1.02 & 88.29 $\to$ 89.49 \; ($+1.20$) \\
ViT-B     & 768  & 1000 & 0.77 & 87.90 $\to$ 88.20 \; ($+0.30$) \\
\bottomrule
\end{tabular}
\end{table}

The gain shrinks as $d/C$ drops below $1$, consistent with the incomplete ETF formation observed for ViT-B above. This suggests a practical rule of thumb: the variance augmentation is expected to help most when the penultimate feature dimension is at least on the order of the number of classes, that is, $d \gtrsim C$.

Beyond this ratio, one could in principle check for a given ID dataset how closely NC1 and NC2 are satisfied using the diagnostic statistics of ~\citet{papyan2020prevalence}: the within-class variability collapse metric $\text{Tr}(\Sigma_W \Sigma_B^\dagger)/C$ for NC1, and the coefficient of variation of centered class-mean norms $\text{Std}_c(\|\mu_c - \mu_G\|_2)/\text{Avg}_c(\|\mu_c - \mu_G\|_2)$ together with the standard deviation of pairwise cosine similarities for NC2. However, this guideline can only assess how well the ID data satisfies the underlying assumptions; it says nothing about which OOD samples will be encountered at test time, since these are arbitrary and unknown in advance.

\clearpage

\section{Proof Theorems}
\label{appendix:proof_theorems}
\setcounter{section}{4}
\setcounter{theorem}{1}
\renewcommand{\thesection}{\arabic{section}}

\begin{theorem}[Bounds on ID Class-wise $\ell_2$ Distance Variance]
\label{theorem:id_variance_bound_appendix}
Let $\phi: \mathcal{X} \rightarrow \mathbb{R}^d$ be the penultimate feature extractor of a pretrained network, where $\mathcal{X}_{in} \subset \mathcal{X}$ denotes the in-distribution input space. Suppose Assumption~\ref{assumption:weak_nc1_nc2} holds for $\phi$ on all $\mathbf{x} \in \mathcal{X}_{in}$. Then for any test ID sample $\mathbf{x} \in \mathcal{X}_{in}$, the class-wise squared $\ell_2$ distance variance satisfies:
\begin{equation*}
\frac{(C-1)K^4}{C^2}\left(1-\gamma\right)^2 \leq \mathrm{Var}_c\left[\|\phi(\mathbf{x}) - \bm{\mu}_c\|_2^2\right] \leq \frac{(C-1)K^4}{C^2}\left(1+\gamma\right)^2 + \frac{2\epsilon^2 K^2(C-2)}{C}
\end{equation*}
where $\gamma \triangleq \epsilon\sqrt{\frac{2C}{K^2(C-1)}}$. In particular, as $\epsilon \to 0$, $\gamma \to 0$ and both bounds converge to $\frac{(C-1)K^4}{C^2}$, which depends solely on the inter-class distance $K$ and the number of classes $C$.
\end{theorem}

\begin{proof}
Let $\mathbf{x} \in \mathcal{X}_{in}$ be a test ID sample from class $c^*$, and denote $\Delta \triangleq \phi(\mathbf{x}) - \bm{\mu}_{c^*}$ with $\|\Delta\|_2 \leq \epsilon$, and $\delta \triangleq \|\Delta\|_2^2 \leq \epsilon^2$, where the first inequality follows from Assumption~\ref{assumption:weak_nc1_nc2} and the second from $\epsilon = \max_c \epsilon_c$. For any $c \neq c^*$, we write $\phi(\mathbf{x}) - \bm{\mu}_c = \Delta + (\bm{\mu}_{c^*} - \bm{\mu}_c)$, so that the squared $\ell_2$ distance expands as:

\begin{equation}
\|\phi(\mathbf{x}) - \bm{\mu}_c\|_2^2 = \underbrace{\|\Delta\|_2^2}_{=\delta} + \underbrace{\|\bm{\mu}_{c^*} - \bm{\mu}_c\|_2^2}_{=K^2 \text{ (by NC2)}} + \underbrace{2\Delta^\top(\bm{\mu}_{c^*}-\bm{\mu}_c)}_{\triangleq \xi_c}
\end{equation}

Letting $\bar{\xi} \triangleq \frac{1}{C-1}\sum_{c \neq c^*} \xi_c$, the mean class-wise squared distance becomes:
\begin{equation}
\bar{d} = \frac{1}{C}\left[\delta + \sum_{c \neq c^*}(\delta + K^2 + \xi_c)\right] = \delta + \frac{(C-1)(K^2 + \bar{\xi})}{C}
\end{equation}
where the first term $\delta$ is the contribution from $c = c^*$, and the second term averages the remaining $C-1$ distances. The deviation from $\bar{d}$ for $c = c^*$ is then:
\begin{equation}
\|\phi(\mathbf{x}) - \bm{\mu}_{c^*}\|_2^2 - \bar{d} = \delta - \delta - \frac{(C-1)(K^2+\bar{\xi})}{C} = -\frac{(C-1)(K^2+\bar{\xi})}{C}
\end{equation}
and for $c \neq c^*$:
\begin{align}
\|\phi(\mathbf{x}) - \bm{\mu}_c\|_2^2 - \bar{d} &= (\delta + K^2 + \xi_c) - \delta - \frac{(C-1)(K^2+\bar{\xi})}{C} \notag \\
&= K^2 + \xi_c - \frac{(C-1)(K^2+\bar{\xi})}{C} \notag \\
&= \frac{C(K^2+\xi_c) - (C-1)(K^2+\bar{\xi})}{C} \notag \\
&= \frac{K^2+\bar{\xi}}{C} + (\xi_c - \bar{\xi})
\end{align}

Letting $S_\xi^2 \triangleq \frac{1}{C-1}\sum_{c \neq c^*}(\xi_c - \bar{\xi})^2 \geq 0$, summing squared deviations over all $C$ classes and using $\sum_{c\neq c^*}(\xi_c - \bar{\xi}) = 0$ by definition of $\bar{\xi}$:

\begin{equation}
\begin{split}
C \cdot \mathrm{Var}_c\left[\|\phi(\mathbf{x}) - \bm{\mu}_c\|_2^2\right] &= \frac{(C-1)^2(K^2+\bar{\xi})^2}{C^2} + (C-1)\left(\frac{K^2+\bar{\xi}}{C}\right)^2 + (C-1)S_\xi^2 \\
&= \frac{(C-1)(K^2+\bar{\xi})^2}{C} + (C-1)S_\xi^2
\end{split}
\end{equation}

Dividing both sides by $C$, the variance reduces to:

\begin{equation}
\mathrm{Var}_c\left[\|\phi(\mathbf{x}) - \bm{\mu}_c\|_2^2\right] = \frac{(C-1)(K^2+\bar{\xi})^2}{C^2} + \frac{(C-1)S_\xi^2}{C}
\label{eq:var_exact_proof}
\end{equation}

We now derive tighter bounds on $\bar{\xi}$ and $S_\xi^2$ by fully exploiting the simplex ETF structure. Since $\bm{\mu}_G = \frac{1}{C}\sum_c \bm{\mu}_c$, we have $\sum_{c \neq c^*}\bm{\mu}_c = C\bm{\mu}_G - \bm{\mu}_{c^*}$, so that:

\begin{equation}
\sum_{c \neq c^*}(\bm{\mu}_{c^*} - \bm{\mu}_c) = (C-1)\bm{\mu}_{c^*} - (C\bm{\mu}_G - \bm{\mu}_{c^*}) = C(\bm{\mu}_{c^*} - \bm{\mu}_G)
\end{equation}

Substituting this into the definition of $\bar{\xi}$ and using linearity of the inner product:

\begin{equation}
\begin{split}
\bar{\xi} &= \frac{1}{C-1}\sum_{c \neq c^*}\xi_c = \frac{2}{C-1}\Delta^\top\sum_{c \neq c^*}(\bm{\mu}_{c^*}-\bm{\mu}_c) \\
&= \frac{2}{C-1}\Delta^\top \cdot C(\bm{\mu}_{c^*} - \bm{\mu}_G) = \frac{2C}{C-1}\Delta^\top(\bm{\mu}_{c^*} - \bm{\mu}_G)
\end{split}
\end{equation}

Denoting $R \triangleq \|\bm{\mu}_c - \bm{\mu}_G\|_2$ for all $c$ (equinorm by NC2), we first express $R$ in terms of $K$. By the simplex ETF condition in Assumption~\ref{assumption:weak_nc1_nc2}, the centered class means satisfy:

\begin{equation}
\frac{(\bm{\mu}_c - \bm{\mu}_G)^\top(\bm{\mu}_{c'} - \bm{\mu}_G)}{\|\bm{\mu}_c - \bm{\mu}_G\|_2\|\bm{\mu}_{c'} - \bm{\mu}_G\|_2} = -\frac{1}{C-1} \quad \forall c \neq c'
\end{equation}

Since $\|\bm{\mu}_c - \bm{\mu}_G\|_2 = R$ for all $c$, this gives:

\begin{equation}
(\bm{\mu}_c - \bm{\mu}_G)^\top(\bm{\mu}_{c'} - \bm{\mu}_G) = -\frac{R^2}{C-1} \quad \forall c \neq c'
\end{equation}

Substituting into the ETF pairwise distance:

\begin{equation}
\begin{split}
K^2 = \|\bm{\mu}_c - \bm{\mu}_{c'}\|_2^2 &= \|(\bm{\mu}_c - \bm{\mu}_G) - (\bm{\mu}_{c'} - \bm{\mu}_G)\|_2^2 \\
&= 2R^2 - 2(\bm{\mu}_c - \bm{\mu}_G)^\top(\bm{\mu}_{c'} - \bm{\mu}_G) \\
&= 2R^2 - 2\left(-\frac{R^2}{C-1}\right) = 2R^2\cdot\frac{C}{C-1}
\end{split}
\end{equation}

which gives $R^2 = \frac{(C-1)K^2}{2C}$. Applying Cauchy-Schwarz to the exact expression of $\bar{\xi}$:

\begin{equation}
|\bar{\xi}| \leq \frac{2C}{C-1}\|\Delta\|_2 \cdot R \leq \frac{2C\epsilon}{C-1} \cdot K\sqrt{\frac{C-1}{2C}} = \epsilon K\sqrt{\frac{2C}{C-1}}
\end{equation}

For $S_\xi^2$, we derive a bound by expanding $\xi_c - \bar{\xi}$. Recall that $\xi_c = 2\Delta^\top(\bm{\mu}_{c^*} - \bm{\mu}_c)$ and $\bar{\xi} = \frac{2C}{C-1}\Delta^\top(\bm{\mu}_{c^*} - \bm{\mu}_G)$. Hence, we can derive following:
\begin{align}
\xi_c - \bar{\xi} &= 2\Delta^\top(\bm{\mu}_{c^*} - \bm{\mu}_c) - \frac{2C}{C-1}\Delta^\top(\bm{\mu}_{c^*} - \bm{\mu}_G) \notag \\
&= 2\Delta^\top\left[(\bm{\mu}_{c^*} - \bm{\mu}_c) - \frac{C}{C-1}(\bm{\mu}_{c^*} - \bm{\mu}_G)\right] \notag \\
&= 2\Delta^\top\left[-\frac{\bm{\mu}_{c^*} - \bm{\mu}_G}{C-1} - (\bm{\mu}_c - \bm{\mu}_G)\right] \notag \\
&= -2\Delta^\top \mathbf{v}_c
\end{align}
where $\mathbf{v}_c \triangleq \frac{\bm{\mu}_{c^*} - \bm{\mu}_G}{C-1} + (\bm{\mu}_c - \bm{\mu}_G)$. 

Using the ETF inner product $(\bm{\mu}_{c^*}-\bm{\mu}_G)^\top(\bm{\mu}_c-\bm{\mu}_G) = -\frac{R^2}{C-1}$ for $c \neq c^*$:

\begin{equation}
\begin{split}
\|\mathbf{v}_c\|_2^2 &= \frac{R^2}{(C-1)^2} - \frac{2R^2}{(C-1)^2} + R^2 \\
&= R^2\left(1 - \frac{1}{(C-1)^2}\right) \\
&= \frac{R^2 C(C-2)}{(C-1)^2} = \frac{K^2(C-2)}{2(C-1)}
\end{split}
\end{equation}

Applying Cauchy-Schwarz to each term $(\xi_c - \bar{\xi})^2 = (-2\Delta^\top \mathbf{v}_c)^2 \leq 4\|\Delta\|_2^2\|\mathbf{v}_c\|_2^2 \leq 4\epsilon^2\|\mathbf{v}_c\|_2^2$ and summing over $c \neq c^*$:
\begin{equation}
S_\xi^2 = \frac{1}{C-1}\sum_{c\neq c^*}(\xi_c - \bar{\xi})^2 \leq \frac{4\epsilon^2}{C-1}\sum_{c\neq c^*}\|\mathbf{v}_c\|_2^2 = 4\epsilon^2 \cdot \frac{K^2(C-2)}{2(C-1)} = \frac{2\epsilon^2 K^2(C-2)}{C-1}.
\end{equation}

To derive the lower bound, we substitute $S_\xi^2 \geq 0$ and $\bar{\xi} \geq -\epsilon K\sqrt{\frac{2C}{C-1}}$ into Eq.~\eqref{eq:var_exact_proof}:
\begin{align}
\mathrm{Var}_c\left[\|\phi(\mathbf{x}) - \bm{\mu}_c\|_2^2\right] &= \frac{(C-1)(K^2 + \bar{\xi})^2}{C^2} + \frac{(C-1)S_\xi^2}{C} \notag \\
&\geq \frac{(C-1)(K^2 + \bar{\xi})^2}{C^2} \notag \\
&\geq \frac{(C-1)\left(K^2 - \epsilon K\sqrt{\frac{2C}{C-1}}\right)^2}{C^2} \notag \\
&= \frac{(C-1)K^4\left(1 - \frac{\epsilon}{K}\sqrt{\frac{2C}{C-1}}\right)^2}{C^2} \notag \\
&= \frac{(C-1)K^4}{C^2}\left(1 - \epsilon\sqrt{\frac{2C}{K^2(C-1)}}\right)^2
\end{align}
where the third inequality holds since $K^2 + \bar{\xi} > 0$ for all admissible $\bar{\xi}$, as $\epsilon < K/2$ implies $K^2 - \epsilon K\sqrt{\frac{2C}{C-1}} \geq K\left(K - \epsilon \cdot 2\right) > 0$, so $(K^2 + \bar{\xi})^2$ is minimized at $\bar{\xi} = -\epsilon K\sqrt{\frac{2C}{C-1}}$. For the upper bound, substituting $\bar{\xi} \leq \epsilon K\sqrt{\frac{2C}{C-1}}$ and $S_\xi^2 \leq \frac{2\epsilon^2 K^2(C-2)}{C-1}$ into Eq.~\eqref{eq:var_exact_proof}:

\begin{equation}
\mathrm{Var}_c\left[\|\phi(\mathbf{x}) - \bm{\mu}_c\|_2^2\right] \leq \frac{(C-1)K^4}{C^2}\left(1 + \epsilon\sqrt{\frac{2C}{K^2(C-1)}}\right)^2 + \frac{2\epsilon^2 K^2(C-2)}{C}
\end{equation}

Finally, as $\epsilon \to 0$, all terms involving $\epsilon$ vanish, and Eq.~\eqref{eq:var_exact_proof} converges to $\frac{(C-1)K^4}{C^2}$, completing the proof.
\end{proof}

\begin{theorem}[Sufficient Condition for Strictly Higher Variance in ID Samples]
\label{theorem:variance_separation_appendix}
Let $\hat{\phi}(\mathbf{x}) = (\phi(\mathbf{x}) - \bm{\mu}_G) / \|\phi(\mathbf{x}) - \bm{\mu}_G\|_2$ be the centered L2-normalized penultimate feature extractor of a pretrained network. Suppose Assumption~\ref{assumption:weak_nc1_nc2} holds for $\hat{\phi}$ on all $\mathbf{x} \in \mathcal{X}_{in}$, where $R = \|\bm{\mu}_c\|_2 \leq 1$ and $\sum_c \bm{\mu}_c = \mathbf{0}$. For any OOD sample $\mathbf{x}^{out} \in \mathcal{X} \setminus \mathcal{X}_{in}$ satisfying $\max_c |\hat{\phi}(\mathbf{x}^{out})^\top \bm{\mu}_c| < \frac{R(R-\epsilon)}{\sqrt{C-1}}$, we have $\mathrm{Var}_c\left[\|\hat{\phi}(\mathbf{x}) - \bm{\mu}_c\|_2^2\right] > \mathrm{Var}_c\left[\|\hat{\phi}(\mathbf{x}^{out}) - \bm{\mu}_c\|_2^2\right]$ for any ID sample $\mathbf{x} \in \mathcal{X}_{in}$.
\end{theorem}
\begin{proof}
We establish the result by deriving a lower bound on $\mathrm{Var}_c\left[\|\hat{\phi}(\mathbf{x}) - \bm{\mu}_c\|_2^2\right]$ for ID samples and an upper bound on $\mathrm{Var}_c\left[\|\hat{\phi}(\mathbf{x}^{out}) - \bm{\mu}_c\|_2^2\right]$ for OOD samples. Since $\|\bm{\mu}_c\|_2 = R$ for all $c$, the ETF inter-class distance satisfies $K^2 = \frac{2R^2C}{C-1}$. Substituting into Theorem~\ref{theorem:id_variance_bound} gives $\frac{(C-1)K^4}{C^2} = \frac{4R^4}{C-1}$ and $\gamma = \frac{\epsilon}{R}$, so that:
\begin{align}
\mathrm{Var}_c\left[\|\hat{\phi}(\mathbf{x}) - \bm{\mu}_c\|_2^2\right] &\geq \frac{(C-1)K^4}{C^2}(1-\gamma)^2 \notag \\
&= \frac{4R^4}{C-1}\left(1 - \frac{\epsilon}{R}\right)^2 \notag \\
&= \frac{4R^4}{C-1} \cdot \frac{(R-\epsilon)^2}{R^2} \notag \\
&= \frac{4R^2(R-\epsilon)^2}{C-1}
\end{align}
For the OOD upper bound, since $\|\hat{\phi}(\mathbf{x}^{out})\|_2 = 1$ and $\|\bm{\mu}_c\|_2 = R$ for all $c$, we have $\|\hat{\phi}(\mathbf{x}^{out}) - \bm{\mu}_c\|_2^2 = 1 + R^2 - 2s_c$ where $s_c \triangleq \hat{\phi}(\mathbf{x}^{out})^\top\bm{\mu}_c$. Since $R^2$ is constant across classes, it vanishes under the variance operator, giving:
\begin{equation}
\mathrm{Var}_c\left[\|\hat{\phi}(\mathbf{x}^{out}) - \bm{\mu}_c\|_2^2\right] = 4\,\mathrm{Var}_c\left[s_c\right]
\end{equation}
From $\sum_c \bm{\mu}_c = \mathbf{0}$ it follows that $\sum_c s_c = \hat{\phi}(\mathbf{x}^{out})^\top \sum_c \bm{\mu}_c = 0$ and $\mathrm{Var}_c[s_c] = \frac{1}{C}\sum_c s_c^2$. To bound $\mathrm{Var}_c[s_c]$, we introduce $t \triangleq \max_c |s_c|$ as the worst-case absolute cosine similarity between the OOD sample and any class mean. Since $|s_c| \leq t$ for all $c$, we have $s_c^2 \leq t^2$ for all $c$, and therefore $\mathrm{Var}_c[s_c] \leq \frac{1}{C}\sum_c t^2 = t^2$, giving:
\begin{equation}
\mathrm{Var}_c\left[\|\hat{\phi}(\mathbf{x}^{out}) - \bm{\mu}_c\|_2^2\right] \leq 4t^2
\end{equation}
The ID lower bound exceeds the OOD upper bound when:
\begin{equation}
\frac{4R^2(R-\epsilon)^2}{C-1} > 4t^2 \iff t < \frac{R(R-\epsilon)}{\sqrt{C-1}}
\end{equation}
which holds under the stated condition $\max_c |s_c| < \frac{R(R-\epsilon)}{\sqrt{C-1}}$, completing the proof.
\end{proof}

\begin{corollary}[Guaranteed Variance Separation under ETF Subspace Alignment]
\label{corollary:variance_separation_always_appendix}
Under the assumptions of Theorem~\ref{theorem:variance_separation}, suppose further that for any ID sample $\mathbf{x} \in \mathcal{X}_{in}$ from class $c^*$, the within-class deviation satisfies $\hat{\phi}(\mathbf{x}) - \bm{\mu}_{c^*} \in \mathrm{span}(\bm{\mu}_1, \ldots, \bm{\mu}_C)$. Then for any OOD sample $\mathbf{x}^{out} \in \mathcal{X} \setminus \mathcal{X}_{in}$, $\mathrm{Var}_c\left[\|\hat{\phi}(\mathbf{x}) - \bm{\mu}_c\|_2^2\right] \geq \mathrm{Var}_c\left[\|\hat{\phi}(\mathbf{x}^{out}) - \bm{\mu}_c\|_2^2\right]$ holds, with equality if and only if $\hat{\phi}(\mathbf{x}^{out}) \in \mathrm{span}(\bm{\mu}_1, \ldots, \bm{\mu}_C)$.
\end{corollary}

\begin{proof}
We first establish the exact expression $\mathrm{Var}_c\left[\|\hat{\phi}(\mathbf{x}) - \bm{\mu}_c\|_2^2\right] = \frac{4R^2\rho}{C-1}$ where $\rho \triangleq \|P\hat{\phi}(\mathbf{x})\|_2^2$ and $P$ denotes the orthogonal projection onto $\mathrm{span}(\bm{\mu}_1, \ldots, \bm{\mu}_C)$. Let $M = [\bm{\mu}_1, \ldots, \bm{\mu}_C]^\top \in \mathbb{R}^{C \times d}$ denote the matrix of class means stacked as rows, and let $\mathbf{1} \in \mathbb{R}^C$ denote the all-ones vector. Since $\|\hat{\phi}(\mathbf{x})\|_2 = 1$ and $\|\bm{\mu}_c\|_2 = R$ for all $c$, we have $d_c^2 = \|\hat{\phi}(\mathbf{x}) - \bm{\mu}_c\|_2^2 = 1 + R^2 - 2s_c$ where $s_c \triangleq \hat{\phi}(\mathbf{x})^\top\bm{\mu}_c$. Since $R^2$ is constant across classes, it vanishes under the variance operator, giving $\mathrm{Var}_c[d_c^2] = 4\mathrm{Var}_c[s_c]$. From $\sum_c \bm{\mu}_c = \mathbf{0}$:
\begin{equation}
\sum_c s_c = \hat{\phi}(\mathbf{x})^\top\sum_c\bm{\mu}_c = 0 \implies \mathrm{Var}_c[s_c] = \frac{1}{C}\sum_c s_c^2 = \frac{1}{C}\hat{\phi}(\mathbf{x})^\top M^\top M\hat{\phi}(\mathbf{x})
\end{equation}
We now derive $M^\top M = \frac{R^2C}{C-1}P$. Computing $MM^\top \in \mathbb{R}^{C \times C}$ using the ETF conditions $(MM^\top)_{c,c} = R^2$ and $(MM^\top)_{c,c'} = -\frac{R^2}{C-1}$ for $c \neq c'$:
\begin{align}
MM^\top &= R^2 I_C + \left(-\frac{R^2}{C-1}\right)(\mathbf{1}\mathbf{1}^\top - I_C) \notag \\
&= R^2 I_C - \frac{R^2}{C-1}\mathbf{1}\mathbf{1}^\top + \frac{R^2}{C-1}I_C \notag \\
&= R^2\left(1 + \frac{1}{C-1}\right)I_C - \frac{R^2}{C-1}\mathbf{1}\mathbf{1}^\top \notag \\
&= \frac{R^2 C}{C-1}I_C - \frac{R^2}{C-1}\mathbf{1}\mathbf{1}^\top \notag \\
&= \frac{R^2 C}{C-1}\left(I_C - \frac{1}{C}\mathbf{1}\mathbf{1}^\top\right) \notag \\
&= \frac{R^2 C}{C-1} H
\end{align}
where $H = I_C - \frac{1}{C}\mathbf{1}\mathbf{1}^\top$ is the centering matrix. We verify that $H$ is idempotent using $(\mathbf{1}\mathbf{1}^\top)^2 = C\mathbf{1}\mathbf{1}^\top$:
\begin{equation}
H^2 = I_C - \frac{2}{C}\mathbf{1}\mathbf{1}^\top + \frac{1}{C^2}(\mathbf{1}\mathbf{1}^\top)^2 = I_C - \frac{2}{C}\mathbf{1}\mathbf{1}^\top + \frac{1}{C^2}\cdot C\mathbf{1}\mathbf{1}^\top = I_C - \frac{1}{C}\mathbf{1}\mathbf{1}^\top = H
\end{equation}
Since $H$ is idempotent and symmetric, $(MM^\top)^+ = \frac{C-1}{R^2C}H$, so that:
\begin{equation}
P = M^\top(MM^\top)^+M = \frac{C-1}{R^2C}M^\top HM = \frac{C-1}{R^2C}\left(M^\top M - \frac{1}{C}M^\top\mathbf{1}\mathbf{1}^\top M\right)
\end{equation}
Since $M^\top\mathbf{1} = \sum_c\bm{\mu}_c = \mathbf{0}$, we have $M^\top\mathbf{1}\mathbf{1}^\top M = \mathbf{0}$, and therefore:
\begin{equation}
P = \frac{C-1}{R^2C}M^\top M \implies M^\top M = \frac{R^2C}{C-1}P
\end{equation}
Plugging this into the expression for $\mathrm{Var}_c[s_c]$:
\begin{equation}
\mathrm{Var}_c[s_c] = \frac{1}{C}\hat{\phi}(\mathbf{x})^\top\cdot\frac{R^2C}{C-1}P\cdot\hat{\phi}(\mathbf{x}) = \frac{R^2\rho}{C-1} \implies \mathrm{Var}_c\left[\|\hat{\phi}(\mathbf{x}) - \bm{\mu}_c\|_2^2\right] = \frac{4R^2\rho}{C-1}
\end{equation}
For any ID sample $\mathbf{x} \in \mathcal{X}_{in}$, the within-class deviation assumption gives $\hat{\phi}(\mathbf{x}) - \bm{\mu}_{c^*} \in \mathrm{span}(\bm{\mu}_1, \ldots, \bm{\mu}_C)$. Since $\bm{\mu}_{c^*}$ trivially lies in $\mathrm{span}(\bm{\mu}_1, \ldots, \bm{\mu}_C)$, it follows that $\hat{\phi}(\mathbf{x}) \in \mathrm{span}(\bm{\mu}_1, \ldots, \bm{\mu}_C)$, hence $P\hat{\phi}(\mathbf{x}) = \hat{\phi}(\mathbf{x})$ and:
\begin{equation}
\rho_{ID} = \|P\hat{\phi}(\mathbf{x})\|_2^2 = \|\hat{\phi}(\mathbf{x})\|_2^2 = 1 \implies \mathrm{Var}_c\left[\|\hat{\phi}(\mathbf{x}) - \bm{\mu}_c\|_2^2\right] = \frac{4R^2}{C-1}
\end{equation}
For any OOD sample $\mathbf{x}^{out} \in \mathcal{X} \setminus \mathcal{X}_{in}$, since $P$ is an orthogonal projection and $\hat{\phi}(\mathbf{x}^{out})$ lies on the unit sphere, $\rho_{OOD} = \hat{\phi}(\mathbf{x}^{out})^\top P\hat{\phi}(\mathbf{x}^{out}) \leq \|\hat{\phi}(\mathbf{x}^{out})\|_2^2 = 1$, and therefore:
\begin{equation}
\mathrm{Var}_c\left[\|\hat{\phi}(\mathbf{x}^{out}) - \bm{\mu}_c\|_2^2\right] = \frac{4R^2\rho_{OOD}}{C-1} \leq \frac{4R^2}{C-1}
\end{equation}
Hence $\mathrm{Var}_c\left[\|\hat{\phi}(\mathbf{x}) - \bm{\mu}_c\|_2^2\right] \geq \mathrm{Var}_c\left[\|\hat{\phi}(\mathbf{x}^{out}) - \bm{\mu}_c\|_2^2\right]$ holds, with equality if and only if $\rho_{OOD} = 1$, i.e., $\hat{\phi}(\mathbf{x}^{out}) \in \mathrm{span}(\bm{\mu}_1, \ldots, \bm{\mu}_C)$, completing the proof.
\end{proof}

\appendix
\setcounter{section}{2} 
\renewcommand{\thesection}{\Alph{section}}
\clearpage

\section{Implementation Details}
\label{appendix:implementation_details}

This section provides implementation details and hyperparameter configurations for all methods evaluated in the main experiments. All experiments are implemented in PyTorch 2.4.1~\citep{paszke2019pytorch}.

\textbf{Mahalanobis and Mahalanobis++.} Class means $\hat{\bm{\mu}}_c$ and a shared tied covariance matrix $\hat{\Sigma}$ are estimated from the penultimate features of the training set. To prevent ill-conditioning, a regularization term $10^{-3} \cdot \mathbf{I}$ is added to the estimated covariance matrix. The OOD score is computed using the inverse of the covariance matrix:
\begin{equation}
    S(\mathbf{x}) = -\min_c\, d_{\mathrm{Maha},c}(\mathbf{x}), \quad d_{\mathrm{Maha},c}(\mathbf{x}) = (\hat{\phi}(\mathbf{x}) - \hat{\bm{\mu}}_c)^\top \hat{\Sigma}^{-1} (\hat{\phi}(\mathbf{x}) - \hat{\bm{\mu}}_c),
\end{equation}
where $\phi(\mathbf{x}) \in \mathbb{R}^d$ denotes the penultimate layer feature of input $\mathbf{x}$, and the terms $\hat{\Sigma}^{-1}\hat{\bm{\mu}}_c$ and $\hat{\bm{\mu}}_c^\top\hat{\Sigma}^{-1}\hat{\bm{\mu}}_c$ are precomputed for inference efficiency. For Mahalanobis++, features are L2-normalized prior to all computations: $\hat{\phi}(\mathbf{x}) = \phi(\mathbf{x}) / \|\phi(\mathbf{x})\|_2$.

\textbf{MahaVar.} MahaVar follows the same preprocessing and class statistics as Mahalanobis++, with the OOD score augmented by a class-wise distance variance term:
\begin{equation}
    S(\mathbf{x}) = -\min_c\, d_{\mathrm{Maha},c}(\hat{\phi}(\mathbf{x})) + \alpha \cdot \mathrm{Var}_c\bigl[d_{\mathrm{Maha},c}(\hat{\phi}(\mathbf{x}))\bigr].
\end{equation}
The hyperparameter $\alpha \geq 0$ controls the relative contribution of the variance term. $\alpha$ is selected via grid search over 26  values in $[0, 10]$ using the validation AUROC provided by OpenOOD v1.5: OpenImage-O for ImageNet settings and TIN for CIFAR settings. The selected values are $\alpha = 0.05$ (ResNet-50), $0.1$ (Swin-B), $0.07$ (ViT-B), $0.015$ (CIFAR-10), and $0.03$ (CIFAR-100). Since the class means $\hat{\bm{\mu}}_c$ and inverse covariance matrix $\hat{\Sigma}^{-1}$ are precomputed and stored, the class-wise distance variance at inference time is computed directly from the per-sample Mahalanobis distances with negligible additional overhead. In practice, we observe that MahaVar achieves inference latency identical to that of Mahalanobis++.

\textbf{MSP.} MSP~\citep{hendrycks2017a} uses the maximum class probability after softmax as the OOD score:
\begin{equation}
    S(\mathbf{x}) = \max_c\, \mathrm{softmax}(\mathbf{W}\phi(\mathbf{x}) + \mathbf{b})_c,
\end{equation}
where $\mathbf{W} \in \mathbb{R}^{C \times d}$ and $\mathbf{b} \in \mathbb{R}^C$ denote the classifier head weight matrix and bias vector, respectively.

\textbf{MaxLogit.} MaxLogit~\citep{hendrycks2022scaling} replaces the softmax with the maximum raw logit:
\begin{equation}
    S(\mathbf{x}) = \max_c\, (\mathbf{W}\phi(\mathbf{x}) + \mathbf{b})_c.
\end{equation}

\textbf{KL-Matching.} KL-Matching~\citep{hendrycks2022scaling} computes class-conditional reference distributions $\mathbf{d}_c = \frac{1}{N_c}\sum_{i:y_i=c}\mathrm{softmax}(\mathbf{W}\phi(\mathbf{x}_i) + \mathbf{b})$ from the training set, and scores a test sample by:
\begin{equation}
    S(\mathbf{x}) = -\min_c\, \mathrm{KL}(\mathbf{p}(\mathbf{x}) \,\|\, \mathbf{d}_c),
\end{equation}
where $\mathbf{p}(\mathbf{x}) = \mathrm{softmax}(\mathbf{W}\phi(\mathbf{x}) + \mathbf{b})$ is the predicted class probability vector.

\textbf{Energy.} The energy score~\citep{liu2020energy} aggregates all class logits via log-sum-exp:
\begin{equation}
    S(\mathbf{x}) = T \cdot \log \sum_c \exp\!\left(\frac{(\mathbf{W}\phi(\mathbf{x}) + \mathbf{b})_c}{T}\right),
\end{equation}
with temperature $T = 1$ fixed throughout.

\textbf{GEN.} GEN~\citep{Liu_2023_CVPR} computes a generalized entropy over the top-$M$ class probabilities:
\begin{equation}
    S(\mathbf{x}) = -\sum_{m \in \mathrm{top}\text{-}M} p_m^\gamma (1-p_m)^\gamma,
\end{equation}
where $p_m$ denotes the $m$-th largest entry of $\mathbf{p}(\mathbf{x})$, and $\gamma$ controls the sharpness of the entropy measure. $M = 100$ is fixed for all settings except CIFAR-10, where $M = 10$. $\gamma$ is selected from $\{0.05, 0.10, 0.20, 0.30, 0.50, 0.70, 1.00\}$ via validation AUROC; the selected values are $\gamma = 0.05$ (ResNet-50), $0.5$ (Swin-B), $0.3$ (ViT-B), $0.2$ (CIFAR-10), and $0.5$ (CIFAR-100).

\textbf{ReAct.} ReAct~\citep{sun2021react} clips penultimate activations at a threshold $\tau$ set to the $p$-th percentile of the training features, prior to energy scoring:
\begin{equation}
    S(\mathbf{x}) = T \cdot \log \sum_c \exp\!\left(\frac{(\mathbf{W}\,\mathrm{clip}(\phi(\mathbf{x}), \tau) + \mathbf{b})_c}{T}\right),
\end{equation}
where $\mathrm{clip}(\phi(\mathbf{x}), \tau)$ truncates each activation at $\tau$ from above. The percentile $p$ is selected from $\{0.90, 0.95, 0.99\}$ via validation AUROC; the selected values are $p = 0.95$ / $\tau = 1.351$ (ResNet-50), $p = 0.90$ / $\tau = 0.376$ (Swin-B), $p = 0.90$ / $\tau = 0.692$ (ViT-B), $p = 0.99$ / $\tau = 1.056$ (CIFAR-10), and $p = 0.99$ / $\tau = 2.278$ (CIFAR-100).

\textbf{ASH-S and SCALE.} ASH-S~\citep{djurisic2023extremely} prunes activations below a threshold $\tau$ and applies exponential rescaling:
\begin{equation}
    \phi_{\mathrm{ASH}}(\mathbf{x}) = \phi(\mathbf{x}) \cdot \mathbf{1}[\phi(\mathbf{x}) > \tau] \cdot \exp\!\left(\frac{\sum_i \phi_i(\mathbf{x})}{\sum_i \phi_i(\mathbf{x})\mathbf{1}[\phi_i(\mathbf{x}) > \tau]}\right),
\end{equation}
followed by energy scoring. SCALE~\citep{xu2024scaling} applies the same exponential rescaling without pruning:
\begin{equation}
    \phi_{\mathrm{SCALE}}(\mathbf{x}) = \phi(\mathbf{x}) \cdot \exp\!\left(\frac{\sum_i \phi_i(\mathbf{x})}{\sum_i \phi_i(\mathbf{x})\mathbf{1}[\phi_i(\mathbf{x}) > \tau]}\right),
\end{equation}
followed by energy scoring. For both methods, the pruning threshold $\tau$ is determined by the $p$-th percentile of each test sample's own activations, with $p$ selected from $\{0.65, 0.70, 0.75, 0.80, 0.85, 0.90, 0.95\}$ via validation AUROC. For ASH-S, the selected values are $p = 0.85$ (ResNet-50), $0.65$ (Swin-B), $0.65$ (ViT-B), $0.65$ (CIFAR-10), and $0.65$ (CIFAR-100). For SCALE, the selected values are $p = 0.85$ (ResNet-50), $0.95$ (Swin-B), $0.65$ (ViT-B), $0.65$ (CIFAR-10), and $0.65$ (CIFAR-100).

\textbf{KNN.} KNN~\citep{sun2022out} scores a test sample by its cosine distance to the $k$-th nearest neighbor in a subsampled training index:
\begin{equation}
    S(\mathbf{x}) = -(1 - \hat{\phi}(\mathbf{x})^\top \hat{\mathbf{z}}_k),
\end{equation}
where $\hat{\mathbf{z}}_k \in \mathbb{R}^d$ is the $k$-th nearest L2-normalized training feature. The index is built from a random 1\% subsample of the training set across all settings. $k$ is selected from $\{1, 3, 5, 10, 20, 50\}$ via validation AUROC; the selected value is $k = 3$ across all ImageNet backbones. For CIFAR-10, $k = 3$, and for CIFAR-100, $k = 1$.

\textbf{NCI and NCI+filter.} NCI~\citep{liu2025detecting} scores a test sample by the alignment between its centered feature and the predicted class weight vector:
\begin{equation}
    S(\mathbf{x}) = \frac{(\phi(\mathbf{x}) - \hat{\bm{\mu}}_G)^\top}{\|\phi(\mathbf{x}) - \hat{\bm{\mu}}_G\|_2} \mathbf{w}_{\hat{c}} + \alpha \|\phi(\mathbf{x})\|_1,
\end{equation}
where $\hat{c} = \arg\max_c (\mathbf{W}\phi(\mathbf{x}) + \mathbf{b})_c$ is the predicted class, $\hat{\bm{\mu}}_G = \frac{1}{N}\sum_i \phi(\mathbf{x}_i)$ is the global training mean, and $\mathbf{w}_{\hat{c}} \in \mathbb{R}^d$ is the weight vector of the predicted class. The weight $\alpha$ on the L1 norm filter term is selected from the same grid as MahaVar via validation AUROC; the selected values are $\alpha = 5 \times 10^{-4}$ (ResNet-50), $1 \times 10^{-4}$ (Swin-B), $3 \times 10^{-4}$ (ViT-B), and $0.0$ for both CIFAR settings. NCI+filter additionally applies a feature norm threshold selected on the validation set.

\textbf{fDBD.} fDBD~\citep{liu2024fast} scores a test sample by its normalized distance to the decision boundaries between the predicted class and all other classes:
\begin{equation}
    S(\mathbf{x}) = \frac{1}{\|\phi(\mathbf{x}) - \hat{\bm{\mu}}_G\|_2} \cdot \frac{1}{C-1}\sum_{c \neq \hat{c}} \frac{|l_{\hat{c}}(\mathbf{x}) - l_c(\mathbf{x})|}{\|\mathbf{w}_{\hat{c}} - \mathbf{w}_c\|_2}
\end{equation}
where $l_c(\mathbf{x}) = \mathbf{w}_c^\top\phi(\mathbf{x}) + b_c$ denotes the logit for class $c$. The pairwise weight distances $\|\mathbf{w}_{\hat{c}} - \mathbf{w}_c\|_2$ are precomputed as a $C \times C$ matrix. fDBD requires no hyperparameter tuning.

\textbf{ViM.} ViM~\citep{wang2022vim} constructs a virtual OOD logit from the residual of the feature against the top-$D$ principal subspace of the training set, computed via randomized SVD on a subsample of 200,000 training features. The virtual logit is scaled to be commensurate with the class logits and appended to the logit vector prior to softmax scoring. The principal subspace dimension $D$ is selected from $\{64, 128, 256\}$ via validation AUROC; the selected value is $D = 256$ across all ImageNet backbones, $D = 128$ for CIFAR-10, and $D = 64$ for CIFAR-100.

\textbf{NNGuide.} NNGuide~\citep{park2023nearest} guides a base confidence score by the average confidence-weighted cosine similarity to the $k$-nearest neighbors in a bank set:
\begin{equation}
    S(\mathbf{x}) = S_{\mathrm{base}}(\mathbf{x}) \cdot G(\mathbf{x}), \quad G(\mathbf{x}) = \frac{1}{k}\sum_{i=1}^{k} s_{(i)}\, \mathrm{sim}(\mathbf{z}_{(i)}, \mathbf{z}),
\end{equation}
where $\mathbf{z} = \phi(\mathbf{x}) / \|\phi(\mathbf{x})\|_2$ is the L2-normalized test feature, 
$S_{\mathrm{base}}(\mathbf{x})$ is the energy score, 
$\mathbf{z}_{(i)} = \phi(\mathbf{x}_{(i)}) / \|\phi(\mathbf{x}_{(i)})\|_2$ is the L2-normalized 
$i$-th nearest neighbor feature from the bank set, 
$s_{(i)} = S_{\mathrm{base}}(\mathbf{x}_{(i)})$ is its energy score, 
and $\mathrm{sim}(\cdot, \cdot)$ denotes cosine similarity. The neighbors are reordered in descending order of confidence-scaled similarity $s_{(i)}\,\mathrm{sim}(\mathbf{z}_{(i)}, \mathbf{z})$. The bank set consists of $\alpha\%$ of features randomly sampled from the training set. Both $\alpha$ and $k$ are selected via validation AUROC, with $\alpha$ chosen from $\{1, 3, 5\}$ and $k$ from $\{1, 3, 5, 10\}$. The selected values are $k = 10$ / $\alpha = 1\%$ (ResNet-50), $k = 10$ / $\alpha = 5\%$ (Swin-B), $k = 10$ / $\alpha = 3\%$ (ViT-B), $k = 1$ / $\alpha = 5\%$ (CIFAR-10), and $k = 10$ / $\alpha = 5\%$ (CIFAR-100).

\textbf{Computing Infrastructure.} All experiments are conducted on a machine equipped with an AMD Ryzen 9 7950X CPU and an NVIDIA RTX 4090 GPU.

\section{Additional Experiments}
\label{appendix:additional_experiments}

\subsection{Full Experimental Results by Dataset and Backbone}
\label{appendix:full_experimental_results}

This section provides detailed per-dataset and per-backbone OOD detection results to complement the aggregated results presented in Section~\ref{sec:experiment}. Specifically, Tables~\ref{tab:imagenet_resnet}--\ref{tab:imagenet_vit} report per-dataset results on ImageNet for ResNet-50, Swin-B, and ViT-B, respectively, and Tables~\ref{tab:cifar10_resnet}--\ref{tab:cifar100_resnet} report per-dataset results on CIFAR-10 and CIFAR-100. MahaVar achieves the best average performance among all Mahalanobis-based methods across all backbone architectures on ImageNet, and outperforms all evaluated baselines on average on Swin-B and ViT-B. On CIFAR-100, MahaVar achieves the best average performance across all evaluated methods, while on CIFAR-10 it remains competitive with the best-performing baselines. Notably, the performance gains on ViT-B are more modest compared to ResNet-50 and Swin-B, consistent with the insufficient penultimate layer dimension ($d = 768 < C - 1 = 999$) that prevents complete ETF formation as discussed in Section~\ref{sec:neural_collapse_explanation}.

\begin{table}[h!]
\centering
\caption{OOD detection performance on ImageNet with ResNet-50. Results are reported in terms of AUROC and FPR@95 for each OOD dataset. The best and second best results are highlighted in \textbf{bold} and \underline{underline}, respectively.}
\label{tab:imagenet_resnet}
\resizebox{\textwidth}{!}{
\begin{tabular}{lcccccccccccc}
\toprule
\multirow{3}{*}{Method} & \multicolumn{4}{c}{Near-OOD} & \multicolumn{6}{c}{Far-OOD} & \multicolumn{2}{c}{\multirow{2}{*}{Avg}} \\
\cmidrule(lr){2-5} \cmidrule(lr){6-11}
& \multicolumn{2}{c}{SSB-Hard} & \multicolumn{2}{c}{NINCO} & \multicolumn{2}{c}{iNaturalist} & \multicolumn{2}{c}{Texture} & \multicolumn{2}{c}{OpenImage-O} & & \\
\cmidrule(lr){2-3} \cmidrule(lr){4-5} \cmidrule(lr){6-7} \cmidrule(lr){8-9} \cmidrule(lr){10-11} \cmidrule(lr){12-13}
& AUROC & FPR@95 & AUROC & FPR@95 & AUROC & FPR@95 & AUROC & FPR@95 & AUROC & FPR@95 & AUROC & FPR@95 \\
\midrule
MSP         & 72.16 & 84.49 & 79.98 & 75.94 & 88.43 & 52.73 & 82.48 & 63.93 & 84.99 & 63.61 & 81.61 & 68.14 \\
MaxLogit    & 72.75 & 83.70 & 80.41 & 76.61 & 91.15 & 50.82 & 88.42 & 51.30 & 89.26 & 57.13 & 84.40 & 63.91 \\
KL-matching & 68.61 & 82.72 & 79.90 & 73.32 & 89.78 & 44.56 & 84.08 & 64.57 & 86.08 & 60.30 & 81.69 & 65.10 \\
Energy      & 72.35 & 83.82 & 79.70 & 77.56 & 90.61 & 53.83 & 88.74 & 49.26 & 89.15 & 57.09 & 84.11 & 64.31 \\
GEN         & 72.19 & 85.75 & 81.80 & 76.52 & 92.45 & 45.79 & 87.59 & 57.27 & 89.35 & 60.29 & 84.67 & 65.12 \\
ReAct       & 73.08 & \underline{80.47} & 81.71 & 70.43 & 96.32 & 19.96 & 92.80 & 38.41 & 91.88 & 41.33 & 87.16 & 50.12 \\
ASH-S       & \underline{74.59} & 80.89 & 84.42 & \underline{63.51} & \underline{97.70} & \underline{12.33} & 97.86 & 10.60 & \underline{93.83} & \underline{30.66} & \underline{89.68} & \underline{39.60} \\
SCALE       & \textbf{77.23} & \textbf{78.00} & \textbf{85.27} & \textbf{61.79} & \textbf{98.00} & \textbf{10.52} & 97.62 & 11.47 & \textbf{93.97} & \textbf{29.92} & \textbf{90.42} & \textbf{38.34} \\
KNN         & 60.41 & 91.33 & 77.79 & 77.17 & 83.63 & 63.34 & 96.48 & 14.92 & 83.97 & 60.13 & 80.46 & 61.38 \\
NCI         & 66.41 & 88.99 & 80.14 & 77.83 & 92.65 & 46.80 & 93.04 & 37.62 & 90.48 & 54.99 & 84.54 & 61.25 \\
NCI+filter  & 71.43 & 83.55 & 83.26 & 68.97 & 96.43 & 20.51 & 96.34 & 18.33 & 93.39 & 36.36 & 88.17 & 45.54 \\
fDBD        & 70.25 & 86.67 & 82.58 & 74.65 & 93.70 & 40.34 & 93.45 & 34.59 & 91.15 & 52.61 & 86.22 & 57.77 \\
ViM         & 65.11 & 90.53 & 78.55 & 77.82 & 89.55 & 59.21 & 97.96 & 9.67 & 90.40 & 51.33 & 84.31 & 57.71 \\
NNGuide     & 73.67 & 80.61 & 81.09 & 70.30 & 94.50 & 29.16 & 93.87 & 25.58 & 91.57 & 41.36 & 86.94 & 49.40 \\
\midrule
Mahalanobis     & 47.34 & 97.52 & 62.46 & 94.81 & 64.16 & 93.81 & 89.71 & 45.08 & 69.01 & 88.10 & 66.54 & 83.86 \\
Mahalanobis++   & 66.34 & 86.21 & 83.11 & 67.39 & 95.13 & 24.87 & \textbf{98.42} & \textbf{7.79} & 92.24 & 39.30 & 87.05 & 45.11 \\
\textbf{MahaVar (Ours)} & 68.72 & 83.59 & \underline{84.44} & 63.90 & 96.90 & 14.57 & \underline{98.37} & \underline{8.02} & 92.77 & 36.25 & 88.24 & 41.27 \\
\bottomrule
\end{tabular}
}
\end{table}

\begin{table}[h!]
\centering
\caption{OOD detection performance on ImageNet with Swin-B. Results are reported in terms of AUROC and FPR@95 for each OOD dataset. The best and second best results are highlighted in \textbf{bold} and \underline{underline}, respectively.}
\label{tab:imagenet_swin}
\resizebox{\textwidth}{!}{
\begin{tabular}{lcccccccccccc}
\toprule
\multirow{3}{*}{Method} & \multicolumn{4}{c}{Near-OOD} & \multicolumn{6}{c}{Far-OOD} & \multicolumn{2}{c}{\multirow{2}{*}{Avg}} \\
\cmidrule(lr){2-5} \cmidrule(lr){6-11}
& \multicolumn{2}{c}{SSB-Hard} & \multicolumn{2}{c}{NINCO} & \multicolumn{2}{c}{iNaturalist} & \multicolumn{2}{c}{Texture} & \multicolumn{2}{c}{OpenImage-O} & & \\
\cmidrule(lr){2-3} \cmidrule(lr){4-5} \cmidrule(lr){6-7} \cmidrule(lr){8-9} \cmidrule(lr){10-11} \cmidrule(lr){12-13}
& AUROC & FPR@95 & AUROC & FPR@95 & AUROC & FPR@95 & AUROC & FPR@95 & AUROC & FPR@95 & AUROC & FPR@95 \\
\midrule
MSP         & 71.30 & 80.20 & 78.42 & 71.66 & 86.46 & 50.79 & 80.12 & 63.33 & 81.76 & 62.34 & 79.61 & 65.66 \\
MaxLogit    & 65.22 & 81.38 & 71.87 & 72.58 & 78.89 & 55.39 & 74.68 & 60.27 & 70.92 & 66.17 & 72.32 & 67.16 \\
KL-matching & 71.07 & 80.05 & 80.76 & 69.11 & 90.05 & 45.72 & 84.42 & 59.86 & 87.57 & 55.93 & 82.77 & 62.13 \\
Energy      & 59.27 & 86.00 & 64.28 & 79.82 & 67.79 & 73.19 & 69.63 & 64.96 & 60.40 & 76.75 & 64.27 & 76.14 \\
GEN         & 74.20 & 78.55 & 82.96 & 64.66 & 92.01 & 34.05 & 86.19 & 50.58 & 87.65 & 48.81 & 84.60 & 55.33 \\
ReAct       & 68.61 & 84.90 & 80.24 & 75.16 & 88.60 & 58.96 & 85.56 & 56.72 & 85.83 & 62.41 & 81.77 & 67.63 \\
ASH-S       & 35.54 & 98.21 & 26.43 & 99.06 & 15.22 & 99.69 & 22.06 & 98.74 & 16.83 & 99.40 & 23.22 & 99.02 \\
SCALE       & 55.97 & 92.86 & 62.51 & 89.69 & 63.10 & 91.18 & 69.55 & 79.61 & 66.71 & 84.48 & 63.57 & 87.56 \\
KNN         & 71.92 & 85.84 & 83.08 & 72.47 & 92.32 & 48.32 & \underline{90.22} & \underline{45.76} & 91.62 & 49.66 & 85.83 & 60.41 \\
NCI         & 70.40 & 84.68 & 81.47 & 72.01 & 91.76 & 47.20 & 87.55 & 55.37 & 90.78 & 49.23 & 84.39 & 61.70 \\
NCI+filter  & 70.45 & 84.54 & 81.53 & 71.93 & 91.81 & 46.80 & 87.62 & 55.33 & 90.79 & 49.22 & 84.44 & 61.56 \\
fDBD        & 69.96 & 85.04 & 81.22 & 72.51 & 91.88 & 47.64 & 88.10 & 55.12 & 91.03 & 48.96 & 84.44 & 61.85 \\
ViM         & 74.46 & 82.03 & 84.54 & 69.26 & 94.74 & 33.74 & 88.78 & 52.29 & 92.83 & 41.52 & 87.07 & 55.77 \\
NNGuide     & 73.13 & 84.10 & 84.93 & 68.22 & 93.58 & 39.85 & \textbf{91.37} & \textbf{41.43} & 92.31 & 47.32 & 87.06 & 56.19 \\
\midrule
Mahalanobis     & 74.62 & 81.88 & 85.08 & 68.73 & 94.55 & 31.21 & 88.38 & 58.80 & 92.76 & 40.91 & 87.08 & 56.31 \\
Mahalanobis++   & \underline{75.66} & \underline{74.45} & \underline{86.40} & \underline{60.16} & \underline{95.71} & \underline{22.43} & 90.21 & 47.81 & \underline{93.48} & \underline{36.45} & \underline{88.29} & \underline{48.26} \\
\textbf{MahaVar (Ours)} & \textbf{77.28} & \textbf{70.26} & \textbf{88.77} & \textbf{55.22} & \textbf{97.06} & \textbf{15.33} & \underline{90.22} & 48.08 & \textbf{94.12} & \textbf{33.58} & \textbf{89.49} & \textbf{44.50} \\
\bottomrule
\end{tabular}
}
\end{table}

\begin{table}[h!]
\centering
\caption{OOD detection performance on ImageNet with ViT-B. Results are reported in terms of AUROC and FPR@95 for each OOD dataset. The best and second best results are highlighted in \textbf{bold} and \underline{underline}, respectively.}
\label{tab:imagenet_vit}
\resizebox{\textwidth}{!}{
\begin{tabular}{lcccccccccccc}
\toprule
\multirow{3}{*}{Method} & \multicolumn{4}{c}{Near-OOD} & \multicolumn{6}{c}{Far-OOD} & \multicolumn{2}{c}{\multirow{2}{*}{Avg}} \\
\cmidrule(lr){2-5} \cmidrule(lr){6-11}
& \multicolumn{2}{c}{SSB-Hard} & \multicolumn{2}{c}{NINCO} & \multicolumn{2}{c}{iNaturalist} & \multicolumn{2}{c}{Texture} & \multicolumn{2}{c}{OpenImage-O} & & \\
\cmidrule(lr){2-3} \cmidrule(lr){4-5} \cmidrule(lr){6-7} \cmidrule(lr){8-9} \cmidrule(lr){10-11} \cmidrule(lr){12-13}
& AUROC & FPR@95 & AUROC & FPR@95 & AUROC & FPR@95 & AUROC & FPR@95 & AUROC & FPR@95 & AUROC & FPR@95 \\
\midrule
MSP         & 69.03 & 84.77 & 78.08 & 73.32 & 88.16 & 51.61 & 85.09 & 57.52 & 84.85 & 59.96 & 81.04 & 65.44 \\
MaxLogit    & 64.45 & 85.72 & 72.39 & 74.21 & 85.25 & 52.38 & 83.79 & 53.91 & 81.61 & 58.82 & 77.50 & 65.01 \\
KL-matching & 68.47 & 83.72 & 81.10 & 72.03 & 90.18 & 49.87 & 86.93 & 60.45 & 87.51 & 59.37 & 82.84 & 65.09 \\
Energy      & 59.16 & 88.16 & 66.05 & 78.45 & 79.26 & 64.06 & 81.22 & 55.66 & 76.49 & 65.05 & 72.44 & 70.27 \\
GEN         & 69.68 & 86.14 & 81.36 & 72.01 & 92.92 & 39.89 & 89.85 & 48.04 & 89.50 & 51.93 & 84.66 & 59.60 \\
ReAct       & 63.19 & 88.30 & 75.47 & 78.60 & 86.00 & 65.22 & 86.66 & 54.28 & 84.23 & 64.88 & 79.11 & 70.26 \\
ASH-S       & 34.41 & 97.91 & 22.87 & 99.42 & 9.75 & 99.96 & 28.09 & 98.78 & 18.19 & 99.60 & 22.66 & 99.13 \\
SCALE       & 56.78 & 88.75 & 61.34 & 79.94 & 73.62 & 69.13 & 78.94 & 56.78 & 72.51 & 67.33 & 68.64 & 72.39 \\
KNN         & 66.54 & 88.96 & 82.22 & 77.75 & 90.53 & 58.95 & 90.57 & 50.02 & 89.10 & 61.45 & 83.79 & 67.43 \\
NCI         & 66.43 & 87.06 & 80.75 & 74.65 & 90.09 & 54.78 & 87.82 & 54.01 & 88.48 & 58.93 & 82.72 & 65.89 \\
NCI+filter  & 66.80 & 86.79 & 80.88 & 74.31 & 90.31 & 53.37 & 88.06 & 53.24 & 88.50 & 58.28 & 82.91 & 65.20 \\
fDBD        & 65.71 & 87.25 & 80.74 & 74.31 & 90.49 & 53.07 & 88.52 & 52.93 & 89.03 & 57.61 & 82.90 & 65.03 \\
ViM         & 69.33 & 79.21 & 84.58 & 61.94 & 95.54 & 25.29 & 90.55 & \textbf{42.97} & 92.03 & 43.28 & 86.41 & 50.54 \\
NNGuide     & 68.20 & 87.62 & 83.63 & 73.85 & 92.35 & 49.16 & \textbf{91.51} & \underline{46.12} & 90.45 & 58.20 & 85.23 & 62.99 \\
\midrule
Mahalanobis     & 71.42 & 75.69 & 86.44 & \textbf{57.77} & 95.99 & \underline{18.56} & 89.33 & 50.27 & 92.33 & \textbf{39.80} & 87.10 & \textbf{48.42} \\
Mahalanobis++   & \underline{72.51} & \underline{74.69} & \underline{87.19} & 58.32 & \underline{96.23} & 19.35 & 90.82 & 50.17 & \underline{92.77} & 41.96 & \underline{87.90} & 48.90 \\
\textbf{MahaVar (Ours)} & \textbf{73.06} & \textbf{73.75} & \textbf{87.54} & \underline{57.98} & \textbf{96.49} & \textbf{18.47} & \underline{90.84} & 51.40 & \textbf{93.09} & \underline{40.65} & \textbf{88.20} & \underline{48.45} \\
\bottomrule
\end{tabular}
}
\end{table}

\begin{table}[h!]
\centering
\caption{OOD detection performance on CIFAR-10 with ResNet-18. Results are reported in terms of AUROC and FPR@95 for each OOD dataset. The best and second best results are highlighted in \textbf{bold} and \underline{underline}, respectively.}
\label{tab:cifar10_resnet}
\resizebox{\textwidth}{!}{
\begin{tabular}{lcccccccccccccc}
\toprule
 & \multicolumn{4}{c}{Near-OOD} & \multicolumn{8}{c}{Far-OOD} & \multicolumn{2}{c}{Avg} \\
\cmidrule(lr){2-5}\cmidrule(lr){6-13}\cmidrule(lr){14-15}
Method & \multicolumn{2}{c}{TIN} & \multicolumn{2}{c}{CIFAR-100} & \multicolumn{2}{c}{MNIST} & \multicolumn{2}{c}{SVHN} & \multicolumn{2}{c}{Texture} & \multicolumn{2}{c}{Places365} & \multicolumn{2}{c}{Avg} \\
 & AUROC & FPR@95 & AUROC & FPR@95 & AUROC & FPR@95 & AUROC & FPR@95 & AUROC & FPR@95 & AUROC & FPR@95 & AUROC & FPR@95 \\
\midrule
MSP & 88.31 & 57.65 & 86.10 & 61.86 & 97.51 & 16.08 & 90.50 & 58.41 & 88.30 & 58.76 & 88.89 & 55.26 & 89.94 & 51.34 \\
MaxLogit & 88.14 & 46.17 & 84.91 & \underline{53.69} & \underline{99.30} & \underline{3.52} & 90.92 & \underline{45.21} & 87.67 & 48.85 & 89.16 & 43.05 & 90.02 & \underline{40.08} \\
KL-matching & 85.52 & 57.47 & 82.83 & 61.80 & 96.82 & 18.10 & 86.53 & 59.37 & 84.63 & 58.87 & 85.67 & 54.59 & 87.00 & 51.70 \\
Energy & 88.22 & 46.16 & 84.96 & 53.83 & \textbf{99.43} & \textbf{3.06} & 91.04 & \textbf{44.37} & 87.74 & 48.78 & 89.28 & 42.46 & 90.11 & \textbf{39.78} \\
GEN & 88.65 & 47.39 & 85.79 & 54.51 & 99.18 & 4.90 & 91.18 & 47.31 & 88.42 & 49.61 & 89.51 & 44.62 & 90.45 & 41.39 \\
ReAct & 87.73 & 48.76 & 84.79 & 56.04 & 99.26 & 4.40 & 88.74 & 53.55 & 86.62 & 53.60 & 91.06 & 41.37 & 89.70 & 42.95 \\
ASH-S & 82.88 & 49.33 & 79.18 & 57.04 & 98.50 & 9.25 & 83.20 & 56.87 & 81.87 & 51.58 & 87.97 & 42.65 & 85.60 & 44.45 \\
SCALE & 84.12 & 49.04 & 80.34 & 56.89 & 98.87 & 6.28 & 84.89 & 55.86 & 82.71 & 52.93 & 88.73 & \underline{40.74} & 86.61 & 43.62 \\
KNN & \textbf{90.81} & 49.02 & \textbf{88.69} & 55.11 & 96.34 & 23.52 & \underline{92.09} & 50.99 & 91.81 & 47.68 & \underline{91.13} & 48.07 & \textbf{91.81} & 45.73 \\
NCI & 89.10 & 51.65 & 87.20 & 56.92 & 94.99 & 33.73 & 89.60 & 57.71 & 91.30 & 48.16 & 90.09 & 50.90 & 90.38 & 49.85 \\
NCI+filter & 89.10 & 51.65 & 87.20 & 56.92 & 94.99 & 33.73 & 89.60 & 57.71 & 91.30 & 48.16 & 90.09 & 50.90 & 90.38 & 49.85 \\
fDBD & 88.47 & 52.48 & 86.35 & 57.57 & 90.70 & 47.32 & 86.19 & 61.54 & 89.51 & 50.99 & 90.33 & 50.31 & 88.59 & 53.37 \\
ViM & 89.09 & 54.86 & 87.19 & 56.99 & 88.33 & 97.36 & \textbf{92.76} & 45.64 & \textbf{94.15} & \textbf{33.55} & 88.91 & 56.50 & 90.07 & 57.48 \\
NNGuide & 88.61 & \textbf{45.40} & 85.27 & \textbf{53.39} & 99.00 & 5.51 & 89.79 & 48.64 & 88.84 & 46.97 & 89.81 & 42.30 & 90.22 & 40.37 \\
\midrule
Mahalanobis & 87.30 & 63.63 & 85.70 & 64.38 & 86.57 & 98.39 & 89.24 & 64.62 & 92.22 & 44.61 & 87.66 & 63.36 & 88.12 & 66.50 \\
Mahalanobis++ & 89.64 & 52.38 & \underline{87.73} & 57.30 & 95.48 & 28.68 & 90.24 & 60.15 & 91.27 & 50.14 & 90.41 & 49.18 & 90.80 & 49.64 \\
MahaVar (Ours) & \underline{90.61} & \underline{45.81} & \underline{87.73} & 54.84 & 98.39 & 8.87 & 85.18 & 83.94 & \underline{92.25} & \underline{44.31} & \textbf{92.53} & \textbf{35.61} & \underline{91.12} & 45.56 \\
\bottomrule
\end{tabular}}
\end{table}

\begin{table}[h!]
\centering
\caption{OOD detection performance on CIFAR-100 with ResNet-18. Results are reported in terms of AUROC and FPR@95 for each OOD dataset. The best and second best results are highlighted in \textbf{bold} and \underline{underline}, respectively.}
\label{tab:cifar100_resnet}
\resizebox{\textwidth}{!}{
\begin{tabular}{lcccccccccccccc}
\toprule
 & \multicolumn{4}{c}{Near-OOD} & \multicolumn{8}{c}{Far-OOD} & \multicolumn{2}{c}{Avg} \\
\cmidrule(lr){2-5}\cmidrule(lr){6-13}\cmidrule(lr){14-15}
Method & \multicolumn{2}{c}{TIN} & \multicolumn{2}{c}{CIFAR-10} & \multicolumn{2}{c}{MNIST} & \multicolumn{2}{c}{SVHN} & \multicolumn{2}{c}{Texture} & \multicolumn{2}{c}{Places365} & \multicolumn{2}{c}{Avg} \\
 & AUROC & FPR@95 & AUROC & FPR@95 & AUROC & FPR@95 & AUROC & FPR@95 & AUROC & FPR@95 & AUROC & FPR@95 & AUROC & FPR@95 \\
\midrule
MSP & 82.22 & 73.43 & 78.54 & 79.97 & 74.45 & 88.53 & 79.37 & 79.29 & 78.07 & 79.61 & 79.61 & 78.75 & 78.71 & 79.93 \\
MaxLogit & 82.74 & 72.88 & 79.20 & 79.40 & 77.82 & 86.21 & 82.33 & 77.22 & 79.36 & 78.72 & 80.06 & 78.70 & 80.25 & 78.85 \\
KL-matching & 82.18 & 72.60 & 77.91 & 79.66 & 75.76 & 83.82 & 80.21 & 75.37 & 78.61 & 75.64 & 79.60 & 77.19 & 79.04 & 77.38 \\
Energy & 82.51 & 72.45 & 79.01 & 79.66 & 78.22 & 84.45 & 82.67 & 76.03 & 79.33 & 77.91 & 79.81 & 79.06 & 80.26 & 78.26 \\
GEN & 83.09 & 72.45 & \textbf{79.33} & 79.50 & 77.11 & 86.25 & 81.87 & 76.91 & 79.51 & 78.46 & 80.36 & 78.62 & 80.21 & 78.70 \\
ReAct & 82.61 & 72.48 & 78.71 & 79.76 & 77.47 & 84.48 & 82.48 & 76.10 & 80.55 & 77.57 & 80.10 & 79.22 & 80.32 & 78.27 \\
ASH-S & 82.27 & 73.94 & 79.17 & \underline{79.34} & 80.08 & 80.06 & 85.84 & 66.58 & 81.79 & 72.80 & 80.77 & 77.06 & 81.65 & 74.96 \\
SCALE & 82.53 & 73.51 & \underline{79.24} & \textbf{79.16} & 79.62 & 81.63 & 85.14 & 68.94 & 81.28 & 74.29 & 80.77 & 76.71 & 81.43 & 75.71 \\
KNN & 81.28 & 76.26 & 75.02 & 82.78 & 76.44 & 84.07 & 78.10 & 78.30 & 81.66 & 70.43 & 77.59 & 80.27 & 78.35 & 78.69 \\
NCI & \underline{83.78} & 72.05 & 78.25 & 81.16 & 78.04 & 87.75 & 82.99 & 76.46 & 84.18 & 71.86 & 81.05 & 77.78 & 81.38 & 77.84 \\
NCI+filter & \underline{83.78} & 72.05 & 78.25 & 81.16 & 78.04 & 87.75 & 82.99 & 76.46 & 84.18 & 71.86 & 81.05 & 77.78 & 81.38 & 77.84 \\
fDBD & \textbf{83.82} & \textbf{71.31} & 78.24 & 80.72 & 77.51 & 87.34 & 81.56 & 77.05 & 82.96 & 72.71 & 80.58 & 77.62 & 80.78 & 77.79 \\
ViM & 77.75 & 82.21 & 70.89 & 87.14 & \textbf{83.44} & \textbf{76.44} & 80.92 & 80.51 & \underline{87.22} & \underline{56.31} & 74.72 & 84.61 & 79.16 & 77.87 \\
NNGuide & 82.94 & 73.83 & 78.60 & 81.90 & 79.70 & 83.12 & 81.54 & 77.06 & 81.79 & 73.05 & 79.78 & 79.55 & 80.72 & 78.08 \\
\midrule
Mahalanobis & 62.90 & 93.96 & 56.45 & 95.72 & 67.81 & 96.46 & 72.81 & 91.93 & 76.96 & 73.16 & 64.18 & 92.85 & 66.85 & 90.68 \\
Mahalanobis++ & 82.81 & 74.64 & 76.97 & 82.07 & \underline{81.09} & \underline{78.24} & \underline{86.92} & \underline{58.43} & 86.81 & 57.11 & \underline{81.89} & \underline{71.51} & \underline{82.75} & \underline{70.33} \\
MahaVar (Ours) & 83.48 & \underline{71.56} & 76.42 & 82.25 & 80.88 & 82.20 & \textbf{88.39} & \textbf{57.00} & \textbf{88.73} & \textbf{51.79} & \textbf{84.89} & \textbf{63.39} & \textbf{83.80} & \textbf{68.03} \\
\bottomrule
\end{tabular}}
\end{table}

\clearpage

\subsection{Comparison with Recent Baselines}
\label{appendix:recent_baseline}
We measure OOD detection performance under the ImageNet ID setting against baselines not included in the main experiments: CADRef \citep{ling2025cadref}, LTS \citep{djurisic2025logit}, OptFS \citep{zhao2024towards}, and AdaSCALE-A/L \citep{regmi2026adascale}, reported in Table~\ref{tab:additional_baselines_resnet50} through Table~\ref{tab:additional_baselines_vitb}. We exclude Catalyst \citep{Hassan_2026_CVPR} and DAVIS \citep{hassan2026davis} from this comparison due to their architecture-dependent implementation. Both methods operate on spatial pre-pooling activation maps, and although DAVIS accommodates hierarchical backbones with a GAP stage such as Swin Transformer, neither extends to architectures such as ViT, which lacks an explicit GAP operation. Table~\ref{tab:additional_baselines_resnet50} shows that MahaVar trails activation-shaping methods such as AdaSCALE and logit-scaling methods such as LTS by a narrow margin in AUROC, mirroring the gap observed against SCALE in the main ResNet-50 results (Table~\ref{tab:imagenet_resnet}). However, Table~\ref{tab:additional_baselines_swinb} and Table~\ref{tab:additional_baselines_vitb} show that MahaVar substantially outperforms these activation- and logit-scaling methods on Transformer-based backbones. A similar pattern holds against CADRef: while MahaVar trails CADRef by 1.35pp in AUROC and 1.10pp in FPR@95 on ResNet-50, it outperforms CADRef by 1.91pp and 1.61pp in AUROC on Swin-B and ViT-B respectively, and by approximately 9pp in FPR@95 on both backbones.

\begin{table}[H]
\centering
\caption{OOD detection performance on ImageNet with ResNet-50 against additional recent baselines. Results are reported in terms of AUROC and FPR@95 for each OOD dataset. The best and second best results are highlighted in \textbf{bold} and \underline{underline}, respectively.}
\label{tab:additional_baselines_resnet50}
\resizebox{\textwidth}{!}{
\begin{tabular}{lcccccccccccc}
\toprule
\multirow{2}{*}{Method} & \multicolumn{4}{c}{Near-OOD} & \multicolumn{6}{c}{Far-OOD} & \multicolumn{2}{c}{Avg} \\
\cmidrule(lr){2-5} \cmidrule(lr){6-11}
& \multicolumn{2}{c}{SSB-Hard} & \multicolumn{2}{c}{NINCO} & \multicolumn{2}{c}{iNaturalist} & \multicolumn{2}{c}{Texture} & \multicolumn{2}{c}{OpenImage-O} & & \\
\cmidrule(lr){2-3} \cmidrule(lr){4-5} \cmidrule(lr){6-7} \cmidrule(lr){8-9} \cmidrule(lr){10-11}
& AUROC & FPR@95 & AUROC & FPR@95 & AUROC & FPR@95 & AUROC & FPR@95 & AUROC & FPR@95 & AUROC & FPR@95 \\
\midrule
CADRef      & 73.93 & 79.91 & \underline{85.17} & \underline{63.05} & \underline{96.91} & 15.88 & \underline{97.95} & \underline{9.71} & 93.96 & 32.31 & 89.59 & \underline{40.17} \\
OptFS       & 71.11 & 83.79 & 82.45 & 70.48 & 96.62 & 18.43 & 96.10 & 21.86 & 92.55 & 39.26 & 87.76 & 46.77 \\
AdaSCALE-A  & \underline{77.96} & \underline{77.85} & 84.13 & 66.52 & 95.26 & 27.95 & 94.30 & 28.84 & \underline{94.04} & 33.49 & 89.14 & 46.93 \\
AdaSCALE-L  & \textbf{78.16} & \textbf{77.19} & 85.00 & 63.70 & 96.48 & 19.79 & 95.97 & 20.39 & \textbf{94.70} & \textbf{29.40} & \underline{90.06} & 42.09 \\
LTS         & 77.04 & 78.08 & \textbf{85.23} & \textbf{61.77} & \textbf{97.85} & \textbf{11.23} & 97.41 & 12.29 & 94.03 & \underline{30.12} & \textbf{90.31} & \textbf{38.70} \\
MahaVar (Ours) & 68.72 & 83.59 & 84.44 & 63.90 & 96.90 & \underline{14.57} & \textbf{98.37} & \textbf{8.02} & 92.77 & 36.25 & 88.24 & 41.27 \\
\bottomrule
\end{tabular}
}
\end{table}

\begin{table}[H]
\centering
\caption{OOD detection performance on ImageNet with Swin-B against additional recent baselines. Results are reported in terms of AUROC and FPR@95 for each OOD dataset. The best and second best results are highlighted in \textbf{bold} and \underline{underline}, respectively.}
\label{tab:additional_baselines_swinb}
\resizebox{\textwidth}{!}{
\begin{tabular}{lcccccccccccc}
\toprule
\multirow{2}{*}{Method} & \multicolumn{4}{c}{Near-OOD} & \multicolumn{6}{c}{Far-OOD} & \multicolumn{2}{c}{Avg} \\
\cmidrule(lr){2-5} \cmidrule(lr){6-11}
& \multicolumn{2}{c}{SSB-Hard} & \multicolumn{2}{c}{NINCO} & \multicolumn{2}{c}{iNaturalist} & \multicolumn{2}{c}{Texture} & \multicolumn{2}{c}{OpenImage-O} & & \\
\cmidrule(lr){2-3} \cmidrule(lr){4-5} \cmidrule(lr){6-7} \cmidrule(lr){8-9} \cmidrule(lr){10-11}
& AUROC & FPR@95 & AUROC & FPR@95 & AUROC & FPR@95 & AUROC & FPR@95 & AUROC & FPR@95 & AUROC & FPR@95 \\
\midrule
CADRef      & \underline{74.99} & \underline{78.54} & \underline{85.48} & \underline{64.82} & \underline{94.03} & \underline{35.58} & \textbf{90.84} & \textbf{44.32} & \underline{92.53} & \underline{45.09} & \underline{87.58} & \underline{53.67} \\
OptFS       & 68.90 & 85.30 & 80.56 & 72.90 & 90.82 & 49.66 & 87.15 & 56.55 & 90.08 & 50.97 & 83.50 & 63.08 \\
AdaSCALE-A  & 69.35 & 82.26 & 79.23 & 70.53 & 86.10 & 54.83 & 79.41 & 58.91 & 82.94 & 62.14 & 79.40 & 65.73 \\
AdaSCALE-L  & 69.29 & 83.18 & 79.05 & 71.96 & 85.14 & 61.89 & 78.43 & 61.59 & 82.33 & 67.01 & 78.85 & 69.13 \\
LTS         & 65.41 & 82.51 & 75.66 & 72.54 & 84.82 & 52.99 & 79.66 & 60.91 & 79.49 & 63.65 & 77.01 & 66.52 \\
MahaVar (Ours) & \textbf{77.28} & \textbf{70.26} & \textbf{88.77} & \textbf{55.22} & \textbf{97.06} & \textbf{15.33} & \underline{90.22} & \underline{48.08} & \textbf{94.12} & \textbf{33.58} & \textbf{89.49} & \textbf{44.50} \\
\bottomrule
\end{tabular}
}
\end{table}

\begin{table}[H]
\centering
\caption{OOD detection performance on ImageNet with ViT-B against additional recent baselines. Results are reported in terms of AUROC and FPR@95 for each OOD dataset. The best and second best results are highlighted in \textbf{bold} and \underline{underline}, respectively.}
\label{tab:additional_baselines_vitb}
\resizebox{\textwidth}{!}{
\begin{tabular}{lcccccccccccc}
\toprule
\multirow{2}{*}{Method} & \multicolumn{4}{c}{Near-OOD} & \multicolumn{6}{c}{Far-OOD} & \multicolumn{2}{c}{Avg} \\
\cmidrule(lr){2-5} \cmidrule(lr){6-11}
& \multicolumn{2}{c}{SSB-Hard} & \multicolumn{2}{c}{NINCO} & \multicolumn{2}{c}{iNaturalist} & \multicolumn{2}{c}{Texture} & \multicolumn{2}{c}{OpenImage-O} & & \\
\cmidrule(lr){2-3} \cmidrule(lr){4-5} \cmidrule(lr){6-7} \cmidrule(lr){8-9} \cmidrule(lr){10-11}
& AUROC & FPR@95 & AUROC & FPR@95 & AUROC & FPR@95 & AUROC & FPR@95 & AUROC & FPR@95 & AUROC & FPR@95 \\
\midrule
CADRef      & \underline{71.02} & \underline{82.69} & \underline{85.14} & \underline{68.65} & \underline{93.85} & \underline{38.28} & \textbf{91.36} & \textbf{46.40} & \underline{91.57} & \underline{50.10} & \underline{86.59} & \underline{57.22} \\
OptFS       & 66.18 & 87.22 & 80.51 & 74.63 & 89.92 & 56.08 & 88.42 & 53.91 & 88.22 & 60.09 & 82.65 & 66.39 \\
AdaSCALE-A  & 68.03 & 85.81 & 76.48 & 73.80 & 84.77 & 59.27 & 82.71 & 57.73 & 83.45 & 60.77 & 79.09 & 67.48 \\
AdaSCALE-L  & 66.87 & 87.34 & 75.30 & 76.56 & 79.84 & 74.55 & 78.85 & 64.05 & 80.76 & 67.06 & 76.32 & 73.91 \\
LTS         & 64.36 & 85.13 & 75.46 & 72.97 & 88.89 & 46.28 & 86.56 & \underline{49.67} & 83.44 & 59.99 & 79.74 & 62.81 \\
MahaVar (Ours) & \textbf{73.06} & \textbf{73.75} & \textbf{87.54} & \textbf{57.98} & \textbf{96.49} & \textbf{18.47} & \underline{90.84} & 51.40 & \textbf{93.09} & \textbf{40.65} & \textbf{88.20} & \textbf{48.45} \\
\bottomrule
\end{tabular}
}
\end{table}

\clearpage

\subsection{Effect of Variance Term and L2-Normalization}
\label{appendix:effect_of_var}
Table~\ref{tab:ablation_component} presents an ablation study on the contribution of the variance term and L2 normalization, where the hyperparameter $\alpha$ for each method is selected using the OpenImage-O validation set provided by OpenOOD v1.5. Adding the variance term to the standard Mahalanobis distance without L2 normalization (Mahalanobis + Var) already yields substantial improvements over the baseline, demonstrating that class-wise distance variance carries meaningful discriminative information for OOD detection. However, Mahalanobis + Var falls short of Mahalanobis++, suggesting that L2 normalization plays an important complementary role in stabilizing the feature geometry. Notably, using the variance term in isolation (Var Only) substantially underperforms both Mahalanobis++ and MahaVar, indicating that variance alone is insufficient as a standalone OOD score. This confirms that the two terms are complementary rather than substitutable: the nearest-class distance provides a strong signal across the bulk of samples, whereas the variance term additionally recovers discriminative power that the nearest-class distance alone does not capture. MahaVar, which combines both L2 normalization and the variance term, achieves the best average performance across all datasets, confirming that the two components are mutually beneficial and that the variance term provides consistent gains on top of the strong L2-normalized Mahalanobis baseline.

\begin{table}[H]
\centering
\caption{Ablation study on the effect of the variance term and L2 normalization on ImageNet benchmark. Results are averaged over three backbone architectures: ResNet-50, Swin-B, and ViT-B. The best results are highlighted in \textbf{bold}.}
\label{tab:ablation_component}
\resizebox{\textwidth}{!}{
\begin{tabular}{lcccc|cccccc|cc}
\toprule
\multirow{3}{*}{Method} & \multicolumn{4}{c|}{Near-OOD} & \multicolumn{6}{c|}{Far-OOD} & \multicolumn{2}{c}{Avg} \\
\cmidrule(lr){2-5} \cmidrule(lr){6-11}
& \multicolumn{2}{c}{SSB-Hard} & \multicolumn{2}{c|}{NINCO} & \multicolumn{2}{c}{iNaturalist} & \multicolumn{2}{c}{Texture} & \multicolumn{2}{c|}{OpenImage-O} & & \\
\cmidrule(lr){2-3} \cmidrule(lr){4-5} \cmidrule(lr){6-7} \cmidrule(lr){8-9} \cmidrule(lr){10-11} \cmidrule(lr){12-13}
& AUROC & FPR@95 & AUROC & FPR@95 & AUROC & FPR@95 & AUROC & FPR@95 & AUROC & FPR@95 & AUROC & FPR@95 \\
\midrule
Mahalanobis        & 64.46 & 85.03 & 77.99 & 73.77 & 84.90 & 47.86 & 89.14 & 51.38 & 84.70 & 56.27 & 80.24 & 62.86 \\
Mahalanobis + Var  & 72.98 & 78.38 & 84.18 & 66.77 & 96.00 & 22.72 & 90.23 & 48.63 & 90.39 & 45.13 & 86.76 & 52.33 \\
Var Only (L2)      & 58.38 & 89.79 & 58.05 & 91.78 & 75.42 & 75.25 & 48.29 & 96.42 & 59.43 & 91.27 & 59.91 & 88.90 \\
Mahalanobis++      & 71.50 & 78.45 & 85.56 & 61.95 & 95.69 & 22.22 & \textbf{93.15} & \textbf{35.26} & 92.83 & 39.24 & 87.75 & 47.42 \\
MahaVar (Ours)     & \textbf{73.02} & \textbf{75.87} & \textbf{86.92} & \textbf{59.03} & \textbf{96.82} & \textbf{16.12} & 93.14 & 35.83 & \textbf{93.33} & \textbf{36.83} & \textbf{88.64} & \textbf{44.74} \\
\bottomrule
\end{tabular}
}
\end{table}

We further formalize this complementarity with a pointwise argument, showing precisely when the combined score benefits from summing the two terms. Let $d_c^2(x) \triangleq \|\hat\phi(x) - \mu_c\|_2^2$ denote the squared distance to class mean $c$, and let $d_{(1)}^2(x) \triangleq \min_c d_c^2(x)$ denote the nearest-class distance. For any $x \in \mathcal{X}$, the identity underlying the proof of Corollary~\ref{corollary:variance_separation_always} holds unconditionally, without any assumption on $x$:
\begin{equation}
\text{Var}_c\!\left[d_c^2(x)\right] = \frac{4R^2}{C-1}\rho(x), \qquad \rho(x) \triangleq \|P\hat\phi(x)\|_2^2 \in [0,1],
\end{equation}
where $P$ denotes the orthogonal projection onto $\text{span}(\mu_1,\ldots,\mu_C)$. Consider the combined score $\tilde S(x) \triangleq -d_{(1)}^2(x) + \alpha \cdot \text{Var}_c[d_c^2(x)]$. For any pair $x_{\text{id}} \in \mathcal{X}_{\text{in}}$, $x_{\text{ood}} \in \mathcal{X} \setminus \mathcal{X}_{\text{in}}$, define
\begin{equation}
\Delta \triangleq d_{(1)}^2(x_{\text{id}}) - d_{(1)}^2(x_{\text{ood}}), \qquad g \triangleq \frac{4R^2}{C-1}(\rho_{\text{id}} - \rho_{\text{ood}}),
\end{equation}
so that $\tilde S(x_{\text{id}}) > \tilde S(x_{\text{ood}}) \iff \alpha \cdot g > \Delta$. Under Theorem~\ref{theorem:variance_separation}'s hypothesis, where $x_{\text{ood}}$ satisfies $\max_c |s_c| < R(R-\epsilon)/\sqrt{C-1}$ with $s_c \triangleq \hat\phi(x_{\text{ood}})^\top \mu_c$, the min-distance term already ranks the pair correctly. Writing $t \triangleq \max_c |s_c|$, the hypothesis states $t < R(R-\epsilon)/\sqrt{C-1}$. Since $s_c \leq |s_c| \leq t$ for every class $c$, we have $\max_c s_c \leq t$, and therefore
\begin{equation}
d_{(1)}^2(x_{\text{ood}}) = 1 + R^2 - 2\max_c s_c \geq 1 + R^2 - 2t > 1 + R^2 - \frac{2R(R-\epsilon)}{\sqrt{C-1}} \geq 1 + R^2 - 2R(R-\epsilon),
\end{equation}
where the second inequality substitutes the hypothesis bound on $t$ and the last uses $\sqrt{C-1} \geq 1$ for $C \geq 2$. Expanding $(R-\epsilon)^2 = R^2 - 2R\epsilon + \epsilon^2$ gives $1+R^2-2R(R-\epsilon) = 1-(R-\epsilon)^2+\epsilon^2$. Since $0 < \epsilon < R \leq 1$ implies $(R-\epsilon)^2 < R^2 \leq 1$, the leading term is strictly positive, so $d_{(1)}^2(x_{\text{ood}}) > \epsilon^2$. The within-class deviation bound in Assumption~\ref{assumption:weak_nc1_nc2} directly gives $d_{(1)}^2(x_{\text{id}}) \leq \epsilon^2$, and combining the two bounds yields $\Delta \leq 0$. Since Theorem~\ref{theorem:variance_separation} also guarantees $g > 0$ in this regime, both terms agree on the correct ordering, and the variance term reinforces a comparison the min-distance term did not need to fix.

In contrast, when $x_{\text{ood}}$ is unconstrained as in Corollary~\ref{corollary:variance_separation_always}, $\Delta$ can be positive: the min-distance term can genuinely fail if $x_{\text{ood}}$ happens to sit closer to some class mean than $x_{\text{id}}$ does. Setting aside Corollary~\ref{corollary:variance_separation_always}'s assumption that $\rho_{\text{id}} = 1$ and treating $\rho_{\text{id}}, \rho_{\text{ood}}$ as general, $g > 0$ reduces to the single condition $\rho_{\text{id}} > \rho_{\text{ood}}$, independent of $\Delta$. On average, this ordering holds in practice, with $\rho_{\text{id}}$ exceeding $\rho_{\text{ood}}$ across all 22 of 22 backbone and dataset combinations for which $D \geq C$, spanning every ID/OOD pair in the OpenOOD v1.5 setting. Whenever $g > 0$, a finite threshold
\begin{equation}
\alpha^*(x_{\text{id}}, x_{\text{ood}}) = \frac{\Delta}{g} = \frac{\Delta}{\frac{4R^2}{C-1}(\rho_{\text{id}} - \rho_{\text{ood}})}
\end{equation}
exists such that $\alpha > \alpha^*$ restores the correct ordering $\tilde S(x_{\text{id}}) > \tilde S(x_{\text{ood}})$, even when $\Delta > 0$. This is the precise sense in which the two terms are complementary by construction: they fail under different, non-overlapping geometric conditions, and $\alpha \cdot g > \Delta$ is exactly the condition under which the variance term compensates for a min-distance failure.

This ordering can be understood through the identity itself: since $\rho(x) \in [0,1]$, class-wise distance variance grows monotonically in $\rho$ from $0$, corresponding to a sample orthogonal to every class mean, up to the ETF-geometry maximum at $\rho=1$. Corollary~\ref{corollary:variance_separation_always}'s assumption $\rho_{\text{id}}=1$ sets the top of this range, so it should be read as the idealized limit of the identity rather than a claim we expect every ID sample to satisfy exactly. What the assumption isolates is the mechanism rather than the exact separation: since the same identity and the same constants $R, C$ apply to both ID and OOD samples, the variance separation the theory predicts reduces entirely to the ordering $\rho_{\text{id}} > \rho_{\text{ood}}$, which is precisely the empirical regularity verified above.

This mechanism also explains why the Var Only ablation in Table~\ref{tab:ablation_component} underperforms both Mahalanobis++ and MahaVar. Since $\rho(x)$ captures only the shape of a sample's alignment with the class-mean subspace and not the magnitude of its distance to the nearest class, using $\rho$ alone can misread a sample that is unambiguously far from every class in raw distance if that sample happens to align well with the subspace. This is a genuine limitation of variance as a standalone score, and it is exactly why the combined score falls back on the min-distance term rather than relying on variance alone.

\subsection{Effect of Higher-Order Moment Terms}
\label{appendix:effect_of_higher_order}

To further investigate the discriminative structure of class-wise Mahalanobis distances, we explore incorporating the skewness of the distance distribution as an additional signal. Specifically, we extend the MahaVar score as follows:
\begin{equation}
S(x) = -\min_{c} d_{\text{Maha},c}(\hat{\phi}(x)) + \alpha \cdot \text{Var}_c[d_{\text{Maha},c}(\hat{\phi}(x))] + \beta \cdot \text{Skew}_c[d_{\text{Maha},c}(\hat{\phi}(x))]
\end{equation}
where $\beta \in \mathbb{R}$ controls the contribution and sign of the skewness term, selected via validation AUROC. As shown in Table~\ref{tab:ablation_skew}, incorporating skewness yields marginal but consistent improvements over MahaVar across most datasets, suggesting that the asymmetric structure of the distance distribution contains additional discriminative information beyond variance alone. Notably, the optimal sign of $\beta$ varies across backbone architectures: a positive $\beta$ is selected for ResNet-50, while a negative $\beta$ is selected for Swin-B and ViT-B, reflecting differences in the skewness structure of class-wise distances across feature geometries. However, since the gains are modest and the optimal sign of $\beta$ is architecture-dependent, we adopt the variance-only formulation as our main method.

\begin{table}[h!]
\centering
\caption{Ablation study on the effect of the skewness term on ImageNet benchmark. Results are averaged over three backbone architectures: ResNet-50, Swin-B, and ViT-B. The selected $\beta$ values are $10$ (ResNet-50), $-10$ (Swin-B), and $-10$ (ViT-B). The best results are highlighted in \textbf{bold}.}
\label{tab:ablation_skew}
\resizebox{\textwidth}{!}{
\begin{tabular}{lcccccccccccc}
\toprule
\multirow{2}{*}{Method} & \multicolumn{2}{c}{SSB-Hard} & \multicolumn{2}{c}{NINCO} & \multicolumn{2}{c}{iNaturalist} & \multicolumn{2}{c}{Texture} & \multicolumn{2}{c}{OpenImage-O} & \multicolumn{2}{c}{Avg} \\
\cmidrule(lr){2-3} \cmidrule(lr){4-5} \cmidrule(lr){6-7} \cmidrule(lr){8-9} \cmidrule(lr){10-11} \cmidrule(lr){12-13}
& AUROC & FPR@95 & AUROC & FPR@95 & AUROC & FPR@95 & AUROC & FPR@95 & AUROC & FPR@95 & AUROC & FPR@95 \\
\midrule
Mahalanobis++           & 71.50 & 78.45 & 85.56 & 61.95 & 95.69 & 22.22 & 93.15 & 35.26 & 92.83 & 39.24 & 87.75 & 47.42 \\
MahaVar (Ours)          & 73.02 & 75.87 & 86.92 & 59.03 & 96.82 & 16.12 & 93.14 & 35.83 & 93.33 & 36.83 & 88.64 & 44.74 \\
Mahalanobis++ + Skew    & 71.58 & 78.46 & 85.70 & 61.89 & 95.71 & 22.10 & \textbf{93.20} & \textbf{35.06} & 92.85 & 39.12 & 87.81 & 47.33 \\
MahaVar + Skew (Ours)   & \textbf{73.13} & \textbf{75.85} & \textbf{87.07} & \textbf{58.84} & \textbf{96.85} & \textbf{15.89} & \textbf{93.20} & 35.54 & \textbf{93.36} & \textbf{36.70} & \textbf{88.72} & \textbf{44.56} \\
\bottomrule
\end{tabular}
}
\end{table}

\subsection{Hyperparameter Sensitivity Analysis}
\label{appendix:alpha_sensitivity}

We analyze the sensitivity of MahaVar to the hyperparameter $\alpha$ controlling the contribution of the variance term. The analysis is conducted on the ImageNet validation set provided by OpenOOD v1.5, using in-distribution ImageNet validation samples and OOD OpenImage-O validation samples. The results are reported in Table~\ref{tab:ablation_alpha}. As $\alpha$ increases from zero, performance improves steadily across all backbone architectures, but eventually degrades as the variance term begins to dominate the nearest-class distance term. The optimal $\alpha$ lies in the range of 0.05 to 0.1 across all backbones, and we recommend practitioners to search within this range for reliable performance. 

\begin{table}[h!]
\centering
\caption{Sensitivity analysis of the hyperparameter $\alpha$ on the ImageNet 
validation set (OpenImage-O), reported in terms of AUROC. 
The best performance for each backbone is highlighted in \textbf{bold}.}
\label{tab:ablation_alpha}
\resizebox{\textwidth}{!}{
\begin{tabular}{lcccccccccccccccccccc}
\toprule
& \multicolumn{19}{c}{$\alpha$} \\
\cmidrule(lr){2-20}
Backbone & 0 & 0.0001 & 0.0003 & 0.0005 & 0.001 & 0.002 & 0.003 & 0.005 & 0.007 & 0.01 & 0.012 & 0.015 & 0.02 & 0.03 & 0.05 & 0.07 & 0.1 & 0.15 & 0.2 \\
\midrule
ResNet-50 & 92.25 & 92.26 & 92.26 & 92.26 & 92.28 & 92.30 & 92.32 & 92.36 & 92.40 & 92.45 & 92.49 & 92.53 & 92.60 & 92.69 & \textbf{92.75} & 92.64 & 92.20 & 90.95 & 89.30 \\
Swin-B    & 93.35 & 93.35 & 93.36 & 93.36 & 93.37 & 93.38 & 93.39 & 93.42 & 93.45 & 93.49 & 93.51 & 93.55 & 93.61 & 93.72 & 93.90 & 94.02 & \textbf{94.10} & 93.97 & 93.54 \\
ViT-B     & 92.28 & 92.28 & 92.29 & 92.29 & 92.29 & 92.30 & 92.31 & 92.32 & 92.33 & 92.35 & 92.37 & 92.38 & 92.41 & 92.46 & 92.52 & \textbf{92.55} & 92.50 & 92.24 & 91.73 \\
\bottomrule
\end{tabular}
}
\end{table}

\subsection{Effect of Distance Metric}
\label{appendix:distance_metric}
Table~\ref{tab:ablation_metric} reports the performance of MahaVar under different distance metrics. Across all metrics, incorporating the variance term consistently improves performance, confirming that class-wise distance variance provides meaningful discriminative information regardless of the choice of distance function. Among all combinations, MahaVar with the Mahalanobis distance achieves the best overall performance. We attribute this to the fact that the Mahalanobis distance accounts for the covariance structure of the feature space, providing a more geometrically faithful measure of class proximity than simple L1 or L2 distances, which in turn leads to the most discriminative OOD detection performance.

\begin{table}[h!]
\centering
\caption{Ablation study on the effect of distance metric on ImageNet benchmark. Results are averaged over three backbone architectures: ResNet-50, Swin-B, and ViT-B. The best results are highlighted in \textbf{bold}.}
\label{tab:ablation_metric}
\resizebox{\textwidth}{!}{
\begin{tabular}{lcccc|cccccc|cc}
\toprule
\multirow{3}{*}{Method} & \multicolumn{4}{c|}{Near-OOD} & \multicolumn{6}{c|}{Far-OOD} & \multicolumn{2}{c}{Avg} \\
\cmidrule(lr){2-5} \cmidrule(lr){6-11}
& \multicolumn{2}{c}{SSB-Hard} & \multicolumn{2}{c|}{NINCO} & \multicolumn{2}{c}{iNaturalist} & \multicolumn{2}{c}{Texture} & \multicolumn{2}{c|}{OpenImage-O} & & \\
\cmidrule(lr){2-3} \cmidrule(lr){4-5} \cmidrule(lr){6-7} \cmidrule(lr){8-9} \cmidrule(lr){10-11} \cmidrule(lr){12-13}
& AUROC & FPR@95 & AUROC & FPR@95 & AUROC & FPR@95 & AUROC & FPR@95 & AUROC & FPR@95 & AUROC & FPR@95 \\
\midrule
L1          & 70.61 & 81.10 & 84.26 & 66.67 & 94.23 & 31.27 & 92.53 & 38.20 & 91.70 & 44.88 & 86.67 & 52.42 \\
L1 + Var    & 72.12 & 77.56 & 85.36 & 63.49 & 95.18 & 26.03 & 92.76 & 38.80 & 92.10 & 43.48 & 87.50 & 49.87 \\
L2          & 71.40 & 79.83 & 84.79 & 64.69 & 95.35 & 24.88 & 92.99 & 35.38 & 92.42 & 41.58 & 87.39 & 49.27 \\
L2 + Var    & 72.37 & 77.68 & 85.55 & 62.18 & 96.06 & 20.62 & \textbf{93.31} & \textbf{34.41} & 92.83 & 39.41 & 88.02 & 46.86 \\
Mahalanobis++ & 71.50 & 78.45 & 85.56 & 61.95 & 95.69 & 22.22 & 93.15 & 35.26 & 92.83 & 39.24 & 87.75 & 47.42 \\
MahaVar (Ours) & \textbf{73.02} & \textbf{75.87} & \textbf{86.92} & \textbf{59.03} & \textbf{96.82} & \textbf{16.12} & 93.14 & 35.83 & \textbf{93.33} & \textbf{36.83} & \textbf{88.64} & \textbf{44.74} \\
\bottomrule
\end{tabular}
}
\end{table}

\subsection{Effect of Top-$K$ Nearest Classes for Variance Computation}
\label{appendix:top_k_effect}
Table~\ref{tab:ablation_topk} reports the AUROC of MahaVar as a function of the number of nearest classes $K$ used for variance computation, evaluated on the ImageNet benchmark across three backbone architectures. While the sharp minimum structure observed in Section~\ref{observation_sharp} suggests that the distance structure near the closest class already carries rich discriminative information, our experiments reveal that incorporating distances to a larger number of classes generally improves performance. When $K$ is small, the variance term provides little additional discriminative information beyond the nearest-class distance, as reflected by $\alpha = 0$ being selected in many settings. This effect is particularly pronounced for ResNet-50, where $\alpha = 0$ is selected for all $K \leq 100$, whereas using all classes ($K = 1000$) leads to substantial AUROC gains. Across all three backbones, AUROC generally improves as $K$ increases, suggesting that incorporating the distance structure of a larger number of classes enriches the variance signal. The exception is ViT-B, where $K = 500$ achieves the best average AUROC rather than $K = 1000$, which we attribute to the insufficient penultimate layer dimension ($d = 768 < C - 1 = 999$) that prevents complete ETF formation as discussed in Section~\ref{sec:neural_collapse_explanation}, causing the distances to farther classes to lose their geometric regularity and degrade the quality of the variance signal. Nevertheless, we report results with $K = 1000$ across all backbones including ViT-B in the main experiments, following a consistent evaluation protocol.

\begin{table}[h!]
\centering
\caption{Ablation study on the number of nearest classes $K$ used for variance 
computation on ImageNet, reported in terms of AUROC, evaluated 
across ResNet-50, Swin-B, and ViT-B. The best $\alpha$ for each 
setting is selected on the validation set, and the best results for each backbone are highlighted in \textbf{bold}.}
\label{tab:ablation_topk}
\resizebox{0.8\textwidth}{!}{
\begin{tabular}{llccccccc}
\toprule
Backbone & $K$ & SSB-Hard & NINCO & iNaturalist & Texture & OpenImage-O & Avg & Best $\alpha$ \\
\midrule
\multirow{10}{*}{ResNet-50}
& 5          & 66.34 & 83.11 & 95.13 & \textbf{98.42} & 92.24 & 87.05 & 0.0000 \\
& 10         & 66.34 & 83.11 & 95.13 & \textbf{98.42} & 92.24 & 87.05 & 0.0000 \\
& 20         & 66.34 & 83.11 & 95.13 & \textbf{98.42} & 92.24 & 87.05 & 0.0000 \\
& 30         & 66.34 & 83.11 & 95.13 & \textbf{98.42} & 92.24 & 87.05 & 0.0000 \\
& 50         & 66.34 & 83.11 & 95.13 & \textbf{98.42} & 92.24 & 87.05 & 0.0000 \\
& 100        & 66.34 & 83.11 & 95.13 & \textbf{98.42} & 92.24 & 87.05 & 0.0000 \\
& 200        & 66.25 & 83.02 & 95.16 & 98.39 & 92.25 & 87.02 & 0.0030 \\
& 400        & 66.33 & 82.79 & 95.61 & 98.28 & 92.36 & 87.07 & 0.0300 \\
& 500        & 66.69 & 83.05 & 95.73 & 98.31 & 92.40 & 87.24 & 0.0300 \\
& all (1000) & \textbf{68.72} & \textbf{84.44} & \textbf{96.90} & 98.37 & \textbf{92.77} & \textbf{88.24} & 0.0500 \\
\midrule
\multirow{10}{*}{Swin-B}
& 5          & 75.66 & 86.40 & 95.71 & 90.21 & 93.48 & 88.29 & 0.0000 \\
& 10         & 75.25 & 86.21 & 95.68 & 90.25 & 93.49 & 88.17 & 0.0010 \\
& 20         & 72.87 & 85.26 & 95.76 & 90.64 & 93.68 & 87.64 & 0.0200 \\
& 30         & 72.05 & 84.99 & 95.96 & 90.90 & 93.85 & 87.55 & 0.0500 \\
& 50         & 71.67 & 84.89 & 96.29 & 91.13 & 94.05 & 87.61 & 0.1000 \\
& 100        & 71.53 & 84.90 & 96.79 & \textbf{91.26} & 94.28 & 87.75 & 0.2000 \\
& 200        & 72.58 & 85.77 & 97.15 & 91.01 & 94.27 & 88.16 & 0.2000 \\
& 400        & 74.81 & 87.30 & 97.30 & 90.66 & 94.22 & 88.86 & 0.1500 \\
& 500        & 75.56 & 87.96 & \textbf{97.50} & 90.58 & \textbf{94.31} & 89.18 & 0.1500 \\
& all (1000) & \textbf{77.28} & \textbf{88.77} & 97.06 & 90.22 & 94.12 & \textbf{89.49} & 0.1000 \\
\midrule
\multirow{10}{*}{ViT-B}
& 5          & 72.51 & 87.19 & 96.23 & 90.82 & 92.77 & 87.90 & 0.0000 \\
& 10         & 72.37 & 87.12 & 96.21 & 90.83 & 92.77 & 87.86 & 0.0005 \\
& 20         & 71.12 & 86.53 & 96.14 & 90.99 & 92.82 & 87.52 & 0.0120 \\
& 30         & 70.56 & 86.28 & 96.20 & 91.06 & 92.87 & 87.40 & 0.0300 \\
& 50         & 70.56 & 86.33 & 96.39 & \textbf{91.08} & 92.96 & 87.46 & 0.0500 \\
& 100        & 70.66 & 86.46 & 96.73 & 91.01 & 93.07 & 87.59 & 0.1000 \\
& 200        & 71.23 & 86.87 & 97.04 & 90.78 & 93.09 & 87.80 & 0.1500 \\
& 400        & 72.39 & 87.62 & 97.02 & 90.74 & 93.10 & 88.17 & 0.1000 \\
& 500        & 72.68 & \textbf{88.00} & \textbf{97.34} & 90.62 & \textbf{93.21} & \textbf{88.37} & 0.1500 \\
& all (1000) & \textbf{73.06} & 87.54 & 96.49 & 90.84 & 93.09 & 88.20 & 0.0700 \\
\bottomrule
\end{tabular}
}
\end{table}

\clearpage

\subsection{Effect of Feature Centering in L2-Normalized Mahalanobis Detection}
\label{appendix:effect_centering}

In Theorem~\ref{theorem:variance_separation} and Corollary~\ref{corollary:variance_separation_always}, the theoretical analysis is conducted with $\hat{\phi}(\mathbf{x}) = (\phi(\mathbf{x}) - \bm{\mu}_G) / \|\phi(\mathbf{x}) - \bm{\mu}_G\|_2$, i.e., features are centered by the global mean $\bm{\mu}_G$ prior to L2 normalization, so that the simplex ETF structure applies directly to the class means $\bm{\mu}_c$ in the normalized space. To assess whether this centering step is beneficial in practice, we compare centered and non-centered variants of both Mahalanobis and MahaVar across all three backbone architectures on ImageNet, with results reported in Table~\ref{tab:centering_ablation}. For Swin-B and ViT-B, centering yields nearly identical performance to the non-centered variant ($\|\bm{\mu}_G\|_2 = 2.03$ and $1.67$, respectively), suggesting that the practical impact of centering is limited in these architectures. For ResNet-50, however, centering leads to a substantial performance degradation, which we attribute to the non-negative feature geometry induced by the terminal ReLU activation, reflected in its substantially larger global mean norm ($\|\bm{\mu}_G\|_2 = 21.5$). Despite the theoretical motivation for centering, we adopt the non-centered L2 normalization in the practical implementation of MahaVar, and leave a full theoretical treatment of the $\bm{\mu}_G \neq \mathbf{0}$ setting as well as a deeper understanding of the performance degradation observed for ResNet as future work.

\begin{table}[h]
\caption{Ablation study on the effect of feature centering on ImageNet benchmark.
Mahalanobis++ (ctr) and MahaVar (ctr) denote variants where the
global mean $\bm{\mu}_G$ is subtracted prior to L2 normalization,
i.e., $\hat{\phi}(\mathbf{x}) = (\phi(\mathbf{x}) - \bm{\mu}_G) / \|\phi(\mathbf{x}) - \bm{\mu}_G\|_2$.
Results are reported in terms of AUROC and FPR@95 across three backbone
architectures. The best results for each backbone are highlighted in bold.}
\label{tab:centering_ablation}
\centering
\resizebox{\textwidth}{!}{
\begin{tabular}{lcccccccccccc}
\toprule
\multirow{2}{*}{Method} & \multicolumn{2}{c}{SSB-Hard} & \multicolumn{2}{c}{NINCO} & \multicolumn{2}{c}{iNaturalist} & \multicolumn{2}{c}{Texture} & \multicolumn{2}{c}{OpenImage-O} & \multicolumn{2}{c}{Avg} \\
\cmidrule(lr){2-3} \cmidrule(lr){4-5} \cmidrule(lr){6-7} \cmidrule(lr){8-9} \cmidrule(lr){10-11} \cmidrule(lr){12-13}
& AUROC & FPR@95 & AUROC & FPR@95 & AUROC & FPR@95 & AUROC & FPR@95 & AUROC & FPR@95 & AUROC & FPR@95 \\
\midrule
\multicolumn{13}{l}{\textit{ResNet-50}} \\
Mahalanobis++          & 66.34 & 86.21 & 83.11 & 67.39 & 95.13 & 24.87 & \textbf{98.42} & \textbf{7.79} & 92.24 & 39.30 & 87.05 & 45.11 \\
MahaVar                & 68.72 & 83.59 & \textbf{84.44} & \textbf{63.90} & \textbf{96.90} & \textbf{14.57} & 98.37 & 8.02 & \textbf{92.77} & \textbf{36.25} & \textbf{88.24} & \textbf{41.27} \\
Mahalanobis++ (ctr)    & 63.50 & 90.86 & 77.76 & 80.08 & 79.79 & 86.29 & 89.26 & 48.55 & 83.20 & 70.65 & 78.70 & 75.28 \\
MahaVar (ctr)          & \textbf{70.02} & \textbf{79.28} & 81.13 & 70.23 & 94.99 & 29.82 & 92.29 & 40.12 & 88.14 & 56.44 & 85.31 & 55.18 \\
\midrule
\multicolumn{13}{l}{\textit{Swin-B}} \\
Mahalanobis++          & 75.66 & 74.45 & 86.40 & 60.16 & 95.71 & 22.43 & 90.21 & 47.81 & 93.48 & 36.45 & 88.29 & 48.26 \\
MahaVar                & \textbf{77.28} & 70.26 & \textbf{88.77} & 55.22 & \textbf{97.06} & 15.33 & 90.22 & 48.08 & \textbf{94.12} & 33.58 & 89.49 & 44.50 \\
Mahalanobis++ (ctr)    & 75.94 & 73.33 & 86.69 & 59.22 & 95.98 & 20.95 & 90.01 & 49.38 & 93.56 & 36.17 & 88.44 & 47.81 \\
MahaVar (ctr)          & 77.19 & \textbf{70.11} & 88.73 & \textbf{54.92} & 97.05 & \textbf{15.18} & \textbf{90.46} & \textbf{46.82} & \textbf{94.12} & \textbf{33.21} & \textbf{89.51} & \textbf{44.05} \\
\midrule
\multicolumn{13}{l}{\textit{ViT-B}} \\
Mahalanobis++          & 72.51 & 74.69 & 87.19 & 58.32 & 96.23 & 19.35 & 90.82 & \textbf{50.17} & 92.77 & 41.96 & 87.90 & 48.90 \\
MahaVar                & \textbf{73.06} & 73.75 & \textbf{87.54} & \textbf{57.98} & \textbf{96.49} & \textbf{18.47} & \textbf{90.84} & 51.40 & \textbf{93.09} & \textbf{40.65} & \textbf{88.20} & \textbf{48.45} \\
Mahalanobis++ (ctr)    & 72.64 & 74.42 & 87.30 & 58.03 & 96.24 & 19.26 & 90.83 & 50.21 & 92.76 & 42.03 & 87.95 & 48.79 \\
MahaVar (ctr)          & 73.00 & \textbf{73.66} & 87.48 & 58.12 & 96.44 & 18.64 & 90.77 & 51.84 & 93.05 & 40.89 & 88.15 & 48.63 \\
\bottomrule
\end{tabular}}
\end{table}

\subsection{A Hyperparameter-Free Scoring}
\label{appendix:vard1}

The additive score in Equation \ref{eq:mahavar} combines the nearest-class distance and the variance term through a hyperparameter $\alpha$, which requires an OOD validation set for tuning. We therefore also consider a ratio-based variant that removes this dependency entirely:
\begin{equation}
S_{\text{ratio}}(x) \triangleq \frac{\text{Var}_c[d_{\text{Maha},c}(x)]}{d_{(1)}(x)}.
\end{equation}

This method is motivated by the relationship that ID samples tend to exhibit a smaller nearest-class distance $d_{(1)}(x)$ while OOD samples tend to exhibit the opposite, whereas the variance term exhibits the reverse tendency. We evaluate this score in Table~\ref{tab:vard1_imagenet}, comparing it against Mahalanobis++ and MahaVar on ImageNet with ResNet-50.

\begin{table}[H]
\centering
\caption{AUROC (\%) of the hyperparameter-free Var/$d_{(1)}$ variant against Mahalanobis++ and MahaVar on ImageNet with ResNet-50. The best results are highlighted in \textbf{bold}.}
\label{tab:vard1_imagenet}
\resizebox{\textwidth}{!}{
\begin{tabular}{lccccccc}
\toprule
\multirow{2}{*}{Method} & \multicolumn{2}{c}{Near-OOD} & \multicolumn{3}{c}{Far-OOD} & \multirow{2}{*}{Avg AUROC} & \multirow{2}{*}{Avg FPR@95} \\
\cmidrule(lr){2-3} \cmidrule(lr){4-6}
& SSB-Hard & NINCO & iNaturalist & Texture & OpenImage-O & & \\
\midrule
Mahalanobis++ & 66.34 & 83.11 & 95.13 & \textbf{98.42} & 92.24 & 87.05 & 45.11 \\
MahaVar & 68.72 & \textbf{84.44} & 96.90 & 98.37 & \textbf{92.77} & 88.24 & 41.27 \\
Var/$d_{(1)}$ ($\alpha$-free) & \textbf{70.60} & 84.33 & \textbf{98.14} & 97.14 & 91.81 & \textbf{88.40} & \textbf{39.96} \\
\bottomrule
\end{tabular}
}
\end{table}

Var/$d_{(1)}$ performs on par with the tuned additive form and in several cases marginally better, despite requiring no hyperparameter search. However, we suspect that a ratio form is inherently harder to analyze rigorously than an additive form: the additive score decomposes cleanly into the two terms Theorem~\ref{theorem:variance_separation} and Corollary~\ref{corollary:variance_separation_always} characterize separately, whereas the ratio couples them nonlinearly, and establishing an analogous ID/OOD separation guarantee for it would require a substantially different derivation than the one we present here. The ratio is also potentially prone to numerical instability as $d_{(1)}(x)$ approaches zero, a regime our current framework does not characterize. We view Var/$d_{(1)}$ as a promising direction and leave a full theoretical treatment to future work.

\subsection{Relation to Prior Class-wise Distance Structure Observations}
\label{appendix:ncdd_comparison}

The general intuition that ID samples cluster closer to their own class centroid while OOD samples remain comparably distant from all centroids has been discussed in prior work, including NCDD~\citep{10.1007/978-3-031-98691-8_14}. However, prior work characterizes this asymmetry primarily through the OOD side, namely that OOD samples maintain roughly similar distances across all class centroids, without establishing the corresponding structure on the ID side. Specifically, whether the non-nearest distances plateau immediately starting from the second-nearest class, as shown in Figure~\ref{fig:cifar_maha_combined}, is a distinct claim from the magnitude of the nearest distance alone and has not been previously established.

To verify that this rank-wise plateau structure carries a genuinely different signal beyond the nearest-class distance, we test two controlled alternatives under the same L2-normalized Mahalanobis features and validation protocol as MahaVar. The first adapts NCDD's core scoring idea into our pipeline by combining the nearest distance with the sum of all non-nearest distances:
\begin{equation}
S(x) = -d_{(1)}(x) + \beta \cdot \sum_{c \neq \arg\min_c d_c(x)} d_c(x).
\end{equation}
The second restricts this combination to only the nearest and second-nearest distances:
\begin{equation}
S(x) = -d_{(1)}(x) + \beta \left( d_{(2)}(x) - d_{(1)}(x) \right).
\end{equation}
For both variants, $\beta$ is selected via the same grid search protocol used for $\alpha$ in MahaVar, on ImageNet with ResNet-50. In both cases, the validation-selected $\beta$ converges to 0, and the resulting scores are numerically identical to Mahalanobis++ (AUROC 87.05 / FPR@95 45.11). This indicates that neither the sum of all non-nearest distances nor the gap between the two closest classes alone carries information beyond the nearest-class distance already exploited by Mahalanobis++. In contrast, the variance term, which aggregates information across all classes rather than one or two, is selected with a non-trivial weight ($\alpha = 0.05$) and yields a clear improvement (AUROC 88.24 / FPR@95 41.27). This confirms that the discriminative signal identified in Section~\ref{sec:obseravtion_sharp} depends on the full class-wise distance structure rather than any single pairwise comparison among the nearest classes.

\section{Limitation}
\label{appendix:limitation}
MahaVar does not achieve the best overall performance on CIFAR-10, where KNN attains a higher AUROC, suggesting that non-parametric approaches may be more effective in settings with few classes. Furthermore, the gains of MahaVar are consistent on average but not uniform across individual OOD datasets, as shown by Texture on ResNet-50 and SVHN on CIFAR-10. In addition, MahaVar combines two complementary signals, and the variance term alone is not sufficient as a standalone score. The hyperparameter $\alpha$ is also selected on a held-out OOD validation split, which is a strong assumption that rarely holds in deployment, and alternatives such as Gaussian noise or self-generated proxy data remain unexplored. Lastly, the theoretical analysis in this work is grounded in Neural Collapse geometry, which assumes well-trained networks operating near the terminal phase of training. Relaxing these assumptions to accommodate more general feature geometries would yield a broader theoretical framework applicable to a wider range of practical settings, and we leave this as a direction for future work.

\section{Broader Impacts}
\label{appendix:broder_impact}
This work proposes a post-hoc OOD detection method aimed at improving the reliability of deep neural networks in safety-critical applications. By enabling models to better identify OOD inputs, the proposed method can contribute to safer deployment of AI systems in domains such as autonomous driving, medical imaging, and industrial inspection. We foresee no significant negative societal impacts from this work.

\clearpage
\input{checklist.tex}

\end{document}

%% file: checklist.tex
\section*{NeurIPS Paper Checklist}

\begin{enumerate}

\item {\bf Claims}
    \item[] Question: Do the main claims made in the abstract and introduction accurately reflect the paper's contributions and scope?
    \item[] Answer: \answerYes{} 
    \item[] Justification: The abstract and introduction explicitly state the three main contributions of the paper, and the scope is consistently aligned with the OOD detection task throughout.

\item {\bf Limitations}
    \item[] Question: Does the paper discuss the limitations of the work performed by the authors?
    \item[] Answer: \answerYes{} 
    \item[] Justification: The limitations of this work are discussed in Appendix~\ref{appendix:limitation}.

\item {\bf Theory assumptions and proofs}
    \item[] Question: For each theoretical result, does the paper provide the full set of assumptions and a complete (and correct) proof?
    \item[] Answer: \answerYes{} 
    \item[] Justification: The full proofs of all theoretical results are provided in Appendix~\ref{appendix:proof_theorems}.

    \item {\bf Experimental result reproducibility}
    \item[] Question: Does the paper fully disclose all the information needed to reproduce the main experimental results of the paper to the extent that it affects the main claims and/or conclusions of the paper (regardless of whether the code and data are provided or not)?
    \item[] Answer: \answerYes{} 
    \item[] Justification: The experimental settings are described in Section~\ref{sec:experiment}, and method details including hyperparameter configurations are provided in Appendix~\ref{appendix:implementation_details}.

\item {\bf Open access to data and code}
    \item[] Question: Does the paper provide open access to the data and code, with sufficient instructions to faithfully reproduce the main experimental results, as described in supplemental material?
    \item[] Answer: \answerYes{} 
    \item[] Justification: The code is provided in the supplementary materials.

\item {\bf Experimental setting/details}
    \item[] Question: Does the paper specify all the training and test details (e.g., data splits, hyperparameters, how they were chosen, type of optimizer) necessary to understand the results?
    \item[] Answer: \answerYes{} 
    \item[] Justification: Experimental settings and details are described in Section~\ref{sec:experiment} and Appendix~\ref{appendix:implementation_details}.

\item {\bf Experiment statistical significance}
    \item[] Question: Does the paper report error bars suitably and correctly defined or other appropriate information about the statistical significance of the experiments?
    \item[] Answer: \answerNo{}
    \item[] Justification: MahaVar is a deterministic post-hoc method that requires no training and involves no stochastic components; given fixed pretrained features and class statistics, all results are fully reproducible without variance across runs.

\item {\bf Experiments compute resources}
    \item[] Question: For each experiment, does the paper provide sufficient information on the computer resources (type of compute workers, memory, time of execution) needed to reproduce the experiments?
    \item[] Answer: \answerYes{} 
    \item[] Justification: The computing environment used in the experiments, including the CPU (AMD Ryzen 9 7950X) and GPU (RTX 4090), is specified in Appendix~\ref{appendix:implementation_details}.
    
\item {\bf Code of ethics}
    \item[] Question: Does the research conducted in the paper conform, in every respect, with the NeurIPS Code of Ethics \url{https://neurips.cc/public/EthicsGuidelines}?
    \item[] Answer: \answerYes{} 
    \item[] Justification: The research conducted in this paper conforms with the NeurIPS Code of Ethics.

\item {\bf Broader impacts}
    \item[] Question: Does the paper discuss both potential positive societal impacts and negative societal impacts of the work performed?
    \item[] Answer: \answerYes{} 
    \item[] Justification: A broader impact statement is provided in Appendix~\ref{appendix:broder_impact}.
    
\item {\bf Safeguards}
    \item[] Question: Does the paper describe safeguards that have been put in place for responsible release of data or models that have a high risk for misuse (e.g., pre-trained language models, image generators, or scraped datasets)?
    \item[] Answer: \answerNA{} 
    \item[] Justification: The paper poses no such risks, as it does not release pre-trained models, scraped datasets, or any components with high risk for misuse.

\item {\bf Licenses for existing assets}
    \item[] Question: Are the creators or original owners of assets (e.g., code, data, models), used in the paper, properly credited and are the license and terms of use explicitly mentioned and properly respected?
    \item[] Answer: \answerYes{} 
    \item[] Justification: All datasets and code libraries used in this paper are properly cited.

\item {\bf New assets}
    \item[] Question: Are new assets introduced in the paper well documented and is the documentation provided alongside the assets?
    \item[] Answer: \answerYes{} 
    \item[] Justification: The code is provided as an anonymized zip file in the supplementary materials.

\item {\bf Crowdsourcing and research with human subjects}
    \item[] Question: For crowdsourcing experiments and research with human subjects, does the paper include the full text of instructions given to participants and screenshots, if applicable, as well as details about compensation (if any)? 
    \item[] Answer: \answerNA{} 
    \item[] Justification: This paper does not involve crowdsourcing or research with human subjects.

\item {\bf Institutional review board (IRB) approvals or equivalent for research with human subjects}
    \item[] Question: Does the paper describe potential risks incurred by study participants, whether such risks were disclosed to the subjects, and whether Institutional Review Board (IRB) approvals (or an equivalent approval/review based on the requirements of your country or institution) were obtained?
    \item[] Answer: \answerNA{} 
    \item[] Justification: This paper does not involve crowdsourcing or research with human subjects.

\item {\bf Declaration of LLM usage}
    \item[] Question: Does the paper describe the usage of LLMs if it is an important, original, or non-standard component of the core methods in this research? Note that if the LLM is used only for writing, editing, or formatting purposes and does \emph{not} impact the core methodology, scientific rigor, or originality of the research, declaration is not required.
    \item[] Answer: \answerNA{} 
    \item[] Justification: Large language models were used solely for writing assistance and editing purposes and did not impact the core methodology, scientific rigor, or originality of the research.

\end{enumerate}